\documentclass[sigconf,authorversion=true]{aamas}  % do not change this line!

\renewcommand\footnotetextcopyrightpermission[1]{} % removes footnote with conference information in first column
\usepackage{booktabs}
\usepackage{times}
\usepackage{xcolor}
\usepackage{soul}
\usepackage[utf8]{inputenc}
\usepackage{amsmath,amssymb,amsthm}
\usepackage{graphicx}
\usepackage{subcaption}
\usepackage[all]{xy}
\usepackage{comment}
\usepackage{todonotes}
\allowdisplaybreaks

\acmDOI{}  % do not change this line!
\acmISBN{}  % do not change this line!
\acmConference{Technical Report}{}{}  % do not change this line!
\acmYear{}  % do not change this line!
\copyrightyear{}  % do not change this line!
\acmPrice{}  % do not change this line!

\newcommand{\args}{\ensuremath{\mathcal{A}}}
\newcommand{\attacks}{\ensuremath{\mathcal{R}}}
\newcommand{\supports}{\ensuremath{\mathcal{S}}}
\newcommand{\probDists}{\ensuremath{\mathcal{P}_\args}}
\newcommand{\probLabs}{\ensuremath{\mathcal{L}_\args}}

\newcommand{\attacker}{\ensuremath{\mathrm{Att}}}
\newcommand{\supporter}{\ensuremath{\mathrm{Sup}}}
\newcommand{\constraints}{\ensuremath{\mathcal{C}}}
\newcommand{\indicator}{\ensuremath{\operatorname{1}}}
\newcommand{\valuations}{\ensuremath{\mathcal{V}}}

\excludecomment{proofversion}

\begin{document}

\title{A Polynomial-time Fragment of Epistemic Probabilistic Argumentation (Technical Report)}

% Single author syntax
%\author{Paper 138}  % put your paper number here!
\author{Nico Potyka}
%\authornote{The secretary disavows any knowledge of this author's actions.}
\affiliation{%
  \institution{University of Osnabr\"{u}ck}
  %\country{Germany}
}
\email{npotyka@uos.de}

\begin{abstract}
Probabilistic argumentation allows reasoning about argumentation problems in a way that is well-founded
by probability theory. However, in practice, this approach can be severely limited by the fact that probabilities are defined  by adding an exponential number of terms. 
We show that this exponential blowup can be avoided in an interesting fragment
of epistemic probabilistic argumentation and that
some computational problems that have been considered intractable
can be solved in polynomial time. We give efficient convex programming formulations for 
these problems and 
explore how far our fragment can be extended without loosing tractability.
\end{abstract}

  % put your semicolon-separated keywords here!

\maketitle

\section{Introduction}

Abstract argumentation \cite{dung1995acceptability} studies the acceptability of arguments 
based on their relationships and abstracted from their content. 
To this end, abstract argumentation problems
can be modeled as graphs, where nodes correspond to arguments and edges to special relations like
attack or support.  In the basic setting
introduced in \cite{dung1995acceptability} only attack relations were considered.
In bipolar argumentation, this framework is extended with support relations \cite{amgoud2004bipolarity,boella2010support,cayrol2013bipolarity,CohenGGS14}.
Another useful extension is to go beyond the classical two-valued view that arguments can only be accepted or rejected.
Examples include 
ranking frameworks that can be based on fixed point equations \cite{besnard2001logic,leite2011social,barringer2012temporal,correia2014efficient} or 
the graph structure \cite{cayrol2005graduality,amgoud2013ranking} and weighted argumentation frameworks 
\cite{baroni2015automatic,rago2016discontinuity,amgoud2017evaluation,mossakowski2018modular,potyka2018Kr}.
Probabilistic argumentation frameworks express uncertainty by building up on probability theory and probabilistic reasoning methods.
Uncertainty can be introduced, for example, over possible worlds, over subgraphs of the argumentation
graph or over classical extensions
\cite{dung2010towards,li2011probabilistic,rienstra2012towards,hunter2014,doder2014probabilistic,polberg2014probabilistic,thimm2017probabilities,KidoO17,rienstra2018probabilistic,ThimmCR18,riveret2018labelling}. 
For the subgraph-based approach, the computational complexity has been studied extensively in
\cite{fazzinga2013complexity,fazzinga2018probabilistic}.

Our focus here is on the epistemic approach to probabilistic argumentation that evolved from
work in \cite{thimm2012probabilistic,hunter2013probabilistic}.
The basic idea is to consider probability functions over possible worlds in order to assign degrees of beliefs to arguments. Here, a possible world is a subset of arguments that are assumed to be accepted in this state.
Based on the relationships between arguments, the possible degrees of beliefs are then restricted by semantical constraints.
For example, the probability of an argument can be bounded from
above based on the probabilities of its attackers or bounded from
below by the probability of its supporters. This is conceptually similar to weighted argumentation frameworks, where attack relations
are supposed to decrease the strength of arguments, whereas support
relations are supposed to increase the strength \cite{baroni2015automatic,amgoud2017evaluation,mossakowski2018modular}.

Two basic computational problems for epistemic probabilistic argumentation have been introduced in \cite{HunterT16}.
The \emph{satisfiability problem} asks whether a given set of semantical constraints over an argumentation graph can be satisfied
by a probability function. The \emph{entailment problem} is to answer queries about the probability of arguments. 
To this end, probability bounds on the probability of the argument are computed based on the probability functions that satisfy
the given semantical constraints.
Based on their close relationship to problems considered in probabilistic reasoning, it has been conjectured that these problems are intractable.
However, as we will explain, both problems can actually be solved in polynomial time. Intuitively, the reason is that the semantical
constraints can only talk about atomic probability statements. For this reason, reasoning with probability functions over possible world 
turns out to be equivalent to reasoning with functions that assign probabilities to arguments directly.
We call these functions probability labellings as they can be seen as generalizations of labellings in classical abstract argumentation \cite{caminada2009logical}
that, intuitively, label arguments as rejected (probability $0$), accepted (probability $1$) or undecided (probability $0.5$).

We explain the epistemic probabilistic argumentation approach from \cite{thimm2012probabilistic,hunter2013probabilistic,HunterT16} in more detail in Section \ref{sec_background} and introduce a slight generalization
of the computational problems considered in \cite{HunterT16}. Even more general variants of these 
problems have been considered in \cite{HunterPT2018Arxiv}, but these variants are too general to obtain
polynomial runtime guarantees as we will explain in Section \ref{sec_constraints} and \ref{sec_queries} .
In Section \ref{sec_algorithms}, we show that reasoning with probability labellings
is equivalent to reasoning with probability functions when only atomic probability statements are considered and use this observation to show that
both the satisfiability and the  entailment problem considered in \cite{HunterT16} and their generalizations can be solved in polynomial time.
We then look at how far we can extend our language towards the language considered in \cite{HunterPT2018Arxiv} by allowing connecting arguments or constraints with logical connectives.
In Section \ref{sec_constraints}, we look at more expressive constraints. We find that the constraint language cannot be extended much further.
If we only allow connecting two arguments or their negation by only conjunction or disjunction in probability statements or if we allow connecting two constraints disjunctively, the satisfiability
problem becomes intractable.
In Section \ref{sec_queries}, we look at more expressive queries.
We cannot avoid an exponential blowup when considering arbitrary queries.
However, we show that when applying the principle of maximum entropy, 
conjunctive queries can still be answered in polynomial time.
In particular, we show that a compact representation of the maximum entropy probability function that satisfies the constraints can be computed in
polynomial time. However, when only atomic constraints are considered, the principle of maximum
entropy implies strong independency assumptions that may yield counterintuitive probabilities.

\section{Background}
\label{sec_background}

We consider bipolar argumentation frameworks (BAFs) $(\args, \attacks, \supports)$
consisting of a set of arguments $\args$, an attack relation $\attacks \subseteq \args \times \args$ and a 
support relation $\supports \subseteq \args \times \args$.
$\attacker(A) = \{B \in \args \mid (B,A) \in \attacks\}$ denotes the set of attackers of an argument $A$
and $\supporter(A) = \{B \in \args \mid (B,A) \in \supports\}$ denotes its supporters.
We visualize bipolar argumentation frameworks as graphs, where arguments are denoted as nodes,
solid edges denote attack relations and dashed edges denote support relations.
Figure \ref{ex_graph_1} shows an example BAF with four arguments $A, B, C, D$.
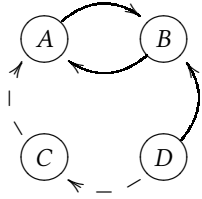
\begin{figure}[tb]
		\scalebox{1.1}{
			\xymatrix{
		%   Line1
*++[o][F-]{A} \ar@{->}@/^1pc/[r]  &*++[o][F-]{B} \ar@{->}@/^1pc/[l]   \\
*++[o][F-]{C} \ar@{-->}@/^1pc/[u] &*++[o][F-]{D}  \ar@{-->}@/^1pc/[l] \ar@{->}@/_1pc/[u]
		}	
	}
\caption{A simple example BAF.\label{ex_graph_1}}
\end{figure}

We define a possible world as a subset of arguments $w \subseteq \args$. Intuitively,
$w$ contains the arguments that are accepted in a particular state of the world. 
As usual, $2^\args$ denotes the set of all
subsets of $\args$, that is, the set of all possible worlds.
An agent can regard certain states of the world more likely than others.
In order to formalize agents' beliefs, we consider probability
functions $P: 2^\args \rightarrow [0,1]$ such that $\sum_{w \in 2^\args} P(w) = 1$.
We denote the set of all probability functions over $\args$ by $\probDists$.
The probability of an argument $A \in \args$ under $P$
is defined by adding the probabilities of all worlds in which $A$ is accepted, that is, 
$P(A) = \sum_{w \in 2^\args, A \in w} P(w)$. $P(A)$ can be understood as a degree of belief
of an agent, where $P(A) = 1$ means complete acceptance and $P(A)=0$ means complete rejectance.

Given an argumentation graph, a probability function should maintain reasonable relationships
between the probabilities of arguments based on their relationships in the graph. 
For example, if an argument is accepted, its attackers should not be accepted.
We introduce some additional terminology for the discussion.
\begin{definition}
Let $P$ be a probability function and let $F$ be a formula over $\args$. We say that
\begin{enumerate}
	\item $P$ (classically) accepts $F$ iff $P(F) > 0.5$ ($P(F) = 1$).
	\item $P$ (classically) rejects $F$ iff $P(F) < 0.5$ ($P(F) = 0$).
\end{enumerate}
\end{definition}
In order to give meaningful semantics to edges in the graph, several constraints have been introduced
in the literature that can be imposed on the probability functions. For the satisfiability and
entailment problem in \cite{HunterT16},
the following constraints have been considered (for attack-only graphs).
\begin{description}
 \item[COH:] $P$ is called \emph{coherent} if for all $A, B \in \args$ with $(A,B) \in \attacks$, we have $P(B) \leq 1 - P(A)$.
 \item[SFOU:] $P$ is called \emph{semi-founded} if $P(A) \geq 0.5$ for all $A \in \args$ with $\attacker(A) = \emptyset$.
 \item[FOU:] $P$ is called \emph{founded} if $P(A) = 1$ for all $A \in \args$ with $\attacker(A) = \emptyset$.
 \item[SOPT:] $P$ is called \emph{semi-optimistic} if $P(A) \geq 1 - \sum_{B \in \attacker(A)} P(B)$
for all $A \in \args$ with $\attacker(A) \neq \emptyset$.
 \item[OPT:] $P$ is called \emph{optimistic} if $P(A) \geq 1 - \sum_{B \in \attacker(A)} P(B)$.
 \item[JUS:] $P$ is called \emph{justifiable} if $P$ is coherent and optimistic.
\end{description}
The intuition for these constraints comes from the idea that probability $0.5$ represents indifference, whereas probabilities smaller (larger) than $0.5$ tend towards rejectance (acceptance) of the argument.
Coherence imposes an upper bound on the beliefs in arguments based on the beliefs 
in their attackers. Semi-Foundedness says that an agent should not tend to reject an argument if there is no reason for this. Foundedness even demands that the argument should be fully accepted in this case.
Semi-optimistic and Optimistic give lower bounds for the belief in an argument that decrease as the belief in its attackers increases. Usually, not all constraints are employed, but a subset is selected that seems reasonable
for a particular application. 
\begin{example}
\label{example_1}
If we demand COH and FOU for the BAF in Figure \ref{ex_graph_1}, we get 
$P(C) = 1$ and $P(D)=1$ from FOU. From COH, we get $P(A) \leq 1 - P(B)$, $P(B) \leq 1-P(A)$ and
$P(B) \leq 1 - P(D)$. Since $P(D)=1$, the last inequality implies $P(B) = 0$.
\end{example}
Inspired by the probabilistic entailment problem from probabilistic logic \cite{Nilsson1986AI,georgakopoulos1988,HansenJaumard2000},
the authors in \cite{HunterT16} considered the following reasoning problems:
Given a partial probability assignment (constraints of the form $P(A) = x$ for some $A \in \args$)
and a subset of the semantical constraints,
\begin{enumerate}
 \item decide whether there is a probability function that satisfies the partial probability assignment and the semantical constraints,
 \item compute lower and upper bounds on the probability of an argument among all probability functions that satisfy the partial probability assignment and the semantical constraints,
 \item decide whether given lower and upper bounds on the probability of an argument are taken by  
  probability functions that satisfy the partial probability assignment and the semantical constraints.
\end{enumerate}
Because of their similarity to intractable probabilistic reasoning problems, 
it has been conjectured that these problems are intractable as well. 
However, as we will explain soon, all three problems can be solved in polynomial time.

\section{Linear Atomic Constraints}

The discussion in \cite{HunterT16} was restricted to attack-only graphs. When we consider
support edges, new constraints are necessary.
In \cite{HunterPT2018Arxiv}, a general constraint language has been introduced
that allows expressing the previous constraints, but also more flexible constraints
that can take account of support relations or of both support and attack relations 
simultaneously.
In particular, constraints can contain complex formulas of arguments and constraints
can be connected via logical connectives. 
Unfortunately, expressiveness hardly ever comes without cost. 
Therefore, we consider only a simple fragment for now.
It still captures the semantical constraints from \cite{HunterT16} that we discussed, 
but also gives us polynomial performance guarantees.
\begin{definition}[Linear Atomic Constraint, Satisfiability]
A \emph{linear atomic constraint} is an expression of the form 
$\sum_{i=1}^n c_i \cdot \pi(A_i) \leq c_0$,
where $A_i \in \args$ and $c_i \in \mathbb{R}$.
A probability function $P$ satisfies a linear atomic constraint iff
$\sum_{i=1}^n c_i \cdot P(A_i) \leq c_0$.
$P$ satisfies a set of linear atomic constraints $\constraints$, denoted as $P \models \constraints$, iff it satisfies all $l \in C$.
In this case, $\constraints$ is called satisfiable.
\end{definition}
Note that, in our notation, $\pi$ is just a syntactic symbol that is used to write constraints. 
$P$ denotes probability functions that may or may not satisfy these constraints.
Note also that constraints with $\geq$ and $=$ can be expressed in our language as well. 
For $\geq$, just note that 
$\sum_{i=1}^n c_i \cdot \pi(A_i) \leq c_0$ is equivalent to
$\sum_{i=1}^n -c_i \cdot \pi(A_i) \geq -c_0$.
For $=$, note that 
$\sum_{i=1}^n c_i \cdot \pi(A_i) \leq c_0$ and
$\sum_{i=1}^n c_i \cdot \pi(A_i) \geq c_0$ together are equivalent to
$\sum_{i=1}^n c_i \cdot \pi(A_i) = c_0$.
We merely restrict our language to constraints with $\leq$ in order to keep the notation simple.
Notice that this restriction is also not important for complexity considerations because the number
of constraints just changes by a constant factor when using $\geq$ and $=$. 

We can now rephrase the computational problems from \cite{HunterT16} for arbitrary linear atomic constraints. 
We consider only the satisfiability and
entailment problem. Since the entailment problem can be solved in polynomial time,
there is no need to look at its decision variant introduced in \cite{HunterT16}.
We also do not need special partial probability assignments, since they can just be expressed by atomic
constraints of the form $\pi(A) = p$.
 Formally, we consider the following computational problems:
\begin{description}
\item[PArgAtSAT:] Given a finite set of linear atomic constraints $C$, decide whether it is satisfiable.
\item[PArgAtENT:] Given a finite set of satisfiable linear atomic constraints $C$ and an argument $A$, compute lower and upper bounds on the probability of $A$ among the probability functions that satisfy $C$.
More precisely, solve the two optimization problems
\begin{align*}
\min_{P \in \probDists}/\max_{P \in \probDists} \quad &P(A) \\
\textit{such that} \quad &P \models \constraints.
\end{align*} 
\end{description}
In the naming scheme, PArg stands for probabilistic argumentation, At for the restriction to linear atomic a
constraints and SAT and ENT stand for satisfiability and entailment, respectively. 
As we mentioned already, the computational problems from \cite{HunterPT2018Arxiv} are indeed a special
case because the partial probability assignments can be encoded as constraints
in $\constraints$. We illustrate this in the following example.
\begin{example}
Consider the BAF in Figure \ref{ex_graph_1}.
Say our partial probability assignment assigns probability $1$ to $B$ and $0$ to $C$.
These assignments correspond to the two linear constraints $\pi(B) = 1$ and $\pi(C) = 0$.
Say we also impose COH. Then, we additionally have the constraints $\pi(A) + \pi(B) \leq 1$ and
$\pi(B) + \pi(D) \leq 1 $. Taken together, these constraints imply that every probability function $P$ 
that satisfies all constraints, must satisfy $P(B)=1$, $P(C)=0$ (partial assignment constraints),
$P(A)=0$ and $P(D) = 0$ (follow with coherence constraints). 
Note that when also adding the foundedness constraints $\pi(C)=1$ and $\pi(D)=1$, the set of constraints becomes unsatisfiable.
\end{example}

For support relations, we can define constraints dual to the attack-only constraints from \cite{HunterT16}. To this end, we replace
$\attacks$ with $\supports$, probability $(1-p)$ with $p$, $\leq$ with $\geq$ and vice versa.
\begin{description}
 \item[S-COH:] $P$ is called \emph{support-coherent} if for all $A, B \in \args$ with $(A,B) \in \supports$, we have $P(B) \geq P(A)$.
 \item[SSCE:] $P$ is called \emph{semi-sceptical} if $P(A) \leq 0.5$ for all $A \in \args$ with $\supporter(A) = \emptyset$.
 \item[SCE:] $P$ is called \emph{sceptical} if $P(A) = 0$ for all $A \in \args$ with $\supporter(A) = \emptyset$.
 \item[SPES:] $P$ is called \emph{semi-pessimistic} if $P(A) \leq \sum_{B \in \supporter(A)} P(B)$
for all $A \in \args$ with $\supporter(A) \neq \emptyset$.
 \item[PES:] $P$ is called \emph{pessimistic} if $P(A) \leq \sum_{B \in \supporter(A)} P(B)$.
\end{description}
Notice that S-COH is dual to COH, SSCE and SCE are dual to SFOU and FOU, SPES and PES are dual to SOPT and OPT. Intuitively, support-coherence says that an argument must be believed at least as strong as its supporter. Semi-scepticality says that an agent should not tend to accept an argument if there is no reason for this. Scepticality demands that the argument should be fully rejected in this case.
Semi-pessimism and Pessimism give upper bounds on the belief in an argument based on the belief in its supporters.
\begin{example}
If we add S-COH to our constraints from Example \ref{example_1}, 
we must have $P(C) \geq P(D)$ and $P(A) \geq P(C)$ for every satisfying probability function $P$.
Since we already know that $P(C)=1$, we can conclude $P(A)=1$.
Overall, the constraints imply $P(A)= P(C) = P(D) = 1$ and $P(B) = 0$.
\end{example}
Note that if both attack and support relations are present, some constraints can be incompatible.
For example, if we employ both foundedness and scepticality, every satisfying probability function
must satisfy $1 = P(A) = 0$ for every unattacked argument, which is clearly impossible.

\section{Polynomial-time Algorithms for Linear Atomic Constraints	}
\label{sec_algorithms}

In this section, we will give polynomial-time algorithms for PArgAtSAT and PArgAtENT.
These algorithms are linear programming algorithms that are commonly applied in this area
\cite{Nilsson1986AI,georgakopoulos1988,HansenJaumard2000}.
Conceptually, the algorithms work as follows:
\begin{enumerate}
	\item Take a knowledge base (and a query) as input and build up a corresponding linear program.
	\item Apply a linear programming solver in order to solve the problem.
\end{enumerate}
Interior-point methods can solve linear programming problems in polynomial time with respect to
the size of the linear program \cite{Boyd2004}. Unfortunately, the linear programs often become
exponentially large for probabilistic reasoning problems because the number of worlds is
exponential in the number of atoms of the language.
However, since our constraint language actually only allows talking about atoms, we can do better.
To this end, we will replace probability functions of exponential size with probability labellings
of linear size.

We define a \emph{probability labelling} as a function $L: \args \rightarrow [0,1]$. That is, a probability labelling
assigns a degree of belief to arguments directly, rather than in an indirect way using possible worlds.
$\probLabs$ denotes the set of all probability labellings over $\args$.
We will now show that probability labellings correspond to equivalence classes of probability functions
and that by restricting to these equivalence classes (represented by probability labellings),
we can solve PArgAtSAT and PArgAtENT in polynomial time.

We call two probability functions $P_1, P_2$ atomically equivalent, denoted as $P_1 \equiv P_2$, 
iff $P_1(A) = P_2(A)$ for all $A \in \args$. Atomic equivalence is an equivalence relation. 
$[P] = \{P' \in \probDists \mid P' \equiv P\}$ denotes the equivalence class of $P$
and $\probDists/\equiv \ = \{[P] \mid P \in \probDists\}$ denotes the set of all equivalence classes.
We first note that there is a one-to-one relationship between $\probDists/\equiv$ and $\probLabs$.
\begin{lemma}
\label{lemma_one_to_one_correspondence}
The function $r: \probDists/\equiv \ \rightarrow \probLabs$ defined by $r([P]) = L_P$, where
$L_P(A) = P(A)$ for all $A \in \args$ is a bijection.  
\end{lemma}
\begin{proof}
First note that $r$ is well-defined: this is because, for all $P' \in [P]$, we have $L_{P'}(A)=P'(A) = P(A) = L_P(A)$ for all $A \in \args$ by definition of $\equiv$. 

$r$ is injective for if $r([P_1]) = r([P_2])$, then $P_1(A) = L_{P_1}(A) = L_{P_2}(A) = P_2(A)$ for
all $A \in \args$. That is, $P_1 \equiv P_2$ and $[P_1] = [P_2]$.  

$r$ is also surjective. To see this, consider an arbitrary $L \in \probLabs$. 
Define $P_L: 2^\args \rightarrow [0,1]$ via $P_L(w) = \prod_{A \in w} L(A) \cdot \prod_{A \in \args \setminus w} (1- L(A))$
for all $w \in 2^\args$. 
We prove by induction over the number of arguments that $\sum_{w \in 2^\args } P_L(w) = 1$.
For the base case, consider $\args = \{A\}$. Then $P_L(\emptyset) + P_L(\{A\}) = \big(1-L(A)\big) + L(A) = 1$.
For the induction step, consider $|\args| = n+1$ and let $B \in \args$.
Then
\begin{align*}
\sum_{w \in 2^\args } P_L(w) 
&=  \sum_{w \in 2^\args } \prod_{A \in w} L(A) \cdot \prod_{A \in \args \setminus w} (1- L(A)) \\
&= \big(1- L(B)\big) \sum_{w \in 2^{\args \setminus \{B\}}} \prod_{A \in w} L(A) \cdot \prod_{A \in \args \setminus w} (1- L(A)) \\
&+ L(B) \sum_{w \in 2^{\args \setminus \{B\}}} \prod_{A \in w} L(A) \cdot \prod_{A \in \args \setminus w} (1- L(A)) \\
&= \big(1- L(B)\big) + L(B) = 1.
\end{align*}
In the second and third row, we partitioned the worlds in those that reject $B$ (second row)
and those that accept $B$ (third row).  
Notice that the sums in the second and third row correspond to possible worlds over a set of arguments
of length $n$, so that our induction hypothesis implies that they sum up to 1.
Hence, $P_L$ is a probability function. 
Furthermore, for all $B \in \args$, we have
\begin{align*}
P_L(B) &= \sum_{w \in 2^\args, B \in w} P(w) \\
&= L(B) \sum_{w \in 2^{\args \setminus \{B\}}} \prod_{A \in w} L(A) \cdot \prod_{A \in \args \setminus w} (1- L(A))
= L(B),
\end{align*}
where we used again the fact that the sum in the second row has to sum up to $1$.
Hence, $r([P_L]) = L$ and $r$ is also surjective and thus bijective.
\end{proof}
Intuitively, $r$ determines a compact representative for the equivalence class $[P]$, namely 
the probability labelling $L_P = r([P])$.
We say that a probability labelling $L$ satisfies a linear atomic constraint
$\sum_{i=1}^n c_i \cdot \pi(A_i) \leq c_0$ iff $\sum_{i=1}^n c_i \cdot L(A_i) \leq c_0$.
The following lemma explains that we can capture the set of all probability functions
that satisfy a constraint by the set of labellings that satisfy the constraint. 
\begin{lemma}
\label{lemma_constraint_equivalence}
The following statements are equivalent:
\begin{enumerate}
	\item P satisfies a linear atomic constraint $l$.
	\item All $P' \in [P]$ satisfy $l$.
	\item  $L_P = r([P])$ satisfies $l$.
\end{enumerate}
\end{lemma}
\begin{proof}
This follows immediately from the satisfaction definition and the observation that $\sum_{i=1}^n c_i \cdot P(A_i) = \sum_{i=1}^n c_i \cdot L_P(A_i)  = \sum_{i=1}^n c_i \cdot P'(A_i)$ for all $P' \in [P]$.
\end{proof}
We can now show that both PArgAtSAT and PArgAtENT can be decided in polynomial time.
To this end, we replace the probability functions with
probability labellings in our optimization
problems and show that this does not change the outcome.
\begin{theorem}
\label{prop_sat_polynomial_time}
PArgAtSAT can be solved in polynomial time. 
In particular, when given $n$ arguments $\args = \{A_1, \dots, A_n\}$ and $m$ constraints $\constraints = \{\sum_{i=1}^n c^{(j)}_i \cdot \pi(A_i) \leq c^{(j)}_0 \mid 1 \leq j \leq m\}$, then $\constraints$ is satisfiable if and only if the linear optimization problem
\begin{align*}
	\min_{(x,s) \in \mathbb{R}^{n+m}} \quad & \sum_{i=1}^m s_i \\
	\textit{such that} \quad &\sum_{i=1}^n c^{(j)}_i \cdot x_i \leq c^{(j)}_0 + s_j, \quad 1 \leq j \leq m, \\
	&0 \leq x \leq 1, s \geq 0,
\end{align*}
has minimum 0.
\end{theorem}
\begin{proof}
First notice that every probablistic labelling $L$ corresponds to a vector
$x \in [0,1]^n$ such that $x_i = L(A_i)$.
The points $(x,s) \in \mathbb{R}^{n+m}$ are intuitively composed of a labelling $x$ and a vector
of slack variables $s$ that relax the constraints.

To begin with, we show that the optimization problem is bounded from below by $0$ and
 always has a well-defined minimum.
Let $s \in \mathbb{R}^m$ be defined by $s_j = \max \{0, -c^{(j)}_0\}$.
Then $(0,s) \in \mathbb{R}^{n+m}$ is a feasible solution because for all constraints, we get
$\sum_{i=1}^n c^{(j)}_i \cdot 0 = 0 \leq c^{(j)}_0 + \max \{0, -c^{(j)}_0\}$.
Hence, the feasible region is non-empty and the theory of linear programming implies that the 
minimum exists \cite{bertsimas97}. In particular, since $s$ is non-negative, it is clear that the minimum can never be smaller than
$0$. 

We show next that the minimum is $0$ if and only if there is a labelling that
satisfies $\constraints$.
Assume first that the minimum is $0$ and let $(x^*,0) \in \mathbb{R}^{n+m}$ be an optimal solution. 
Consider $L^*$ defined by $L^*(A_i) = x^*_i$.
We have $\sum_{i=1}^n c^{(j)}_i \cdot L^*(A_i) 
= \sum_{i=1}^n c^{(j)}_i \cdot x^*_i \leq c^{(j)}_0 + 0$ for all $1 \leq j \leq m$.
Hence,  $L^*$ satisfies $\constraints$.

Conversely, assume that there is a probability labelling that satisfies $\constraints$.
Let $x^* \in [0,1]^n$ be defined by $x^*_i = L(A_i)$ and consider the point $(x^*,0) \in \mathbb{R}^{n+m}$.
We have 
$\sum_{i=1}^n c^{(j)}_i \cdot x^*_i = \sum_{i=1}^n c^{(j)}_i \cdot L^*(A_i) \leq c^{(j)}_0$.
Therefore, $(x^*,0)$ is a feasible solution. In particular, it yields $0$ for the objective function
and hence is minimal.

We know from the theory of linear programming that linear optimization problems can be solved in polynomial
time with respect to the number of optimization variables and constraints \cite{bertsimas97}.
We have $n+m$ optimization variables and $m$ constraints (non-negativity constraints are free). Hence, we can decide in polynomial time whether there exists a probability labelling that satisfies
$\constraints$. If there is such a labelling $L$, then the probability function $P_L$ from the proof of Lemma \ref{lemma_one_to_one_correspondence} satisfies $\constraints$ according to Lemma \ref{lemma_constraint_equivalence}. Conversely, if there is no probability 
function that satisfies $\constraints$, then there can be no labelling that satisfies it either.
For if there was such a labelling $L$, then $P_L$ would satisfy  $\constraints$ as well.
Hence, $\constraints$ can be satisfied by a probability function $P$ if and only if 
the minimum of our linear optimization problem is $0$. Hence, PArgAtSAT can be solved in polynomial time
by the given linear program. 
\end{proof}
\begin{theorem}
PArgAtENT can be solved in polynomial time. 
In particular, when given $n$ arguments $\args = \{A_1, \dots, A_n\}$ and $m$ constraints $\constraints = \{\sum_{i=1}^n c^{(j)}_i \cdot \pi(A_i) \leq c^{(j)}_0 \mid 1 \leq j \leq m\}$ such that
$\constraints$ is satisfiable, then the lower and upper bounds on the probability of $A_k$ are the results of the following linear optimization problems:
\begin{align*}
	\min_{x \in \mathbb{R}^{n}}/ \max_{x \in \mathbb{R}^{n}} \quad & x_k \\
	\textit{such that} \quad &\sum_{i=1}^n c^{(j)}_i \cdot x_i \leq c^{(j)}_0, \quad 1 \leq j \leq m, \\
	&0 \leq x \leq 1.
\end{align*} 
\end{theorem}
\begin{proof}
For concreteness and w.l.o.g. assume that we want to compute bounds on the probability of $A_1$.
We look only at the minimization problem for computing the lower bound (for the maximization problem,
everything is completely analogous). That is, we consider the following linear optimization problem:
\begin{align*}
	\min_{x \in \mathbb{R}^{n}} \quad & x_1 \\
	\textit{such that} \quad &\sum_{i=1}^n c^{(j)}_i \cdot x_i \leq c^{(j)}_0, \quad 1 \leq j \leq m, \\
	&0 \leq x \leq 1.
\end{align*} 
By assumption, $\constraints$ is satisfiable. Hence, the feasible region is non-empty and the
theory of linear programming implies that the minimum exists and can be computed in polynomial time \cite{bertsimas97}.  
The minimum found corresponds exactly to the smallest probability that is assigned to $A_1$ 
by a probability labelling that satisfies the constraints. To see this, note that
if we take a minimal solution $x^*$, we can construct a labelling $L$ that satisfies the constraints
as in the previous proof. In particular, $x^*_1 = L(A_1)$.
There can be no probability labelling that satisfies the constraints and assigns a smaller probability
to $A_1$ because each such labelling yields a feasible
vector $x \in \mathbb{R}^n$ with $x_1 = L(A_1)$.

Similar to before, it follows that the minimum also corresponds to the smallest probability that is assigned to $A_1$ by a probability function that satisfies $\constraints$. 
If the minimum is taken by a labelling $L$, we know that the corresponding probability function $P_L$ 
from the proof of Lemma \ref{lemma_one_to_one_correspondence} yields the same probability and satisfies $\constraints$ according to Lemma \ref{lemma_constraint_equivalence}.
Hence, the minimum cannot be smaller than the probability taken by probability functions that satisfy $\constraints$. 
Conversely, if there is a probability function that satisfies $\constraints$ and gives $P(A_1) = p$,
then the labelling $L_P$ gives $L_P(A_1) = p$ as well and satisfies $\constraints$ according to Lemma \ref{lemma_constraint_equivalence}. Hence, the minimum cannot be larger than the probability taken by probability functions that satisfy $\constraints$ either, and so it must be indeed equal.
Hence, PArgAtENT can be solved in polynomial time by the given linear program. 
\end{proof}

\section{Complex Formulas and k-th order Labellings}

Until now, we looked only at probabilities of arguments. However, the real power of probability functions 
is that they allow computing probabilities for arbitrary formulas over arguments.
By a formula over a set of arguments $\args$, we mean an expression that is formed by connecting the arguments
in $\args$ via logical connectives. Satisfaction of formulas by possible worlds is explained in the usual
recursive way. For example, $w \models \neg F$ iff $w$ does not satisfy $F$ and $w \models F \wedge G$ iff $w$
satisfies both $F$ and $G$.  
We will now also allow non-atomic linear constraints.
\begin{definition}[Linear Constraint]
A \emph{linear constraint} is an expression of the form 
$\sum_{i=1}^n c_i \cdot \pi(F_i) \leq c_0$,
where $F_i$ is a formula over $\args$ and $c_i \in \mathbb{R}$.
\end{definition}
In order to define satisfaction of non-atomic linear constraints, we have to define 
the probability of a formula $F$ under a probability function $P$.
This probability is defined by adding the probabilities of all worlds that satisfy $F$, that is, 
$P(F) = \sum_{w \in 2^\args, w \models F} P(w)$.
$P$ satisfies a linear atomic constraint $\sum_{i=1}^n c_i \cdot \pi(F_i) \leq c_0$ iff
$\sum_{i=1}^n c_i \cdot P(F_i) \leq c_0$.
The following lemma summarizes some well-known computation rules that will be useful in the following.
\begin{lemma}
\label{lemma_computation_rules}
Let $P$ be a probability function and let $F, G$ denote formulas over $\args$.
\begin{enumerate}
	\item $P(F \wedge G) + P(F \wedge \neg G) = P(F)$.
	\item If $F$ entails $G$, then $P(F) \leq P(G)$.
\end{enumerate}
\end{lemma}
\begin{proof}
1. We have 
\begin{align*}
&P(F \wedge G) + P(F \wedge \neg G) \\
&\ = \sum_{w \in 2^\args, w \models (F \wedge G)} P(w) \ + \sum_{w \in 2^\args, w \models (F \wedge \neg G)} P(w) \\
&\ = \sum_{w \in 2^\args, w \models \big((F \wedge G) \vee (F \wedge \neg G)\big)} P(w) \\
&= P(F)
\end{align*}
where we used the fact that $(F \wedge G)$ and $(F \wedge \neg G)$ are exlusive formulas
and that $(F \wedge G) \vee (F \wedge \neg G) \equiv F \wedge (G \vee \neg G) \equiv F$. 

2. Since $F$ entails $G$, every world that satisfies $F$ also satisfies $G$ and therefore
$P(F) = \sum_{w \in 2^\args, w \models F} P(w) \leq  \sum_{w \in 2^\args, w \models G} P(w) = P(G)$.
\end{proof}
Our language is now significantly more expressive.
Unfortunately, computing $P(F)$ naively involves adding an exponential number of terms.
Is there a more efficient way to compute $P(F)$?
As we saw before, if $F$ is an atom, we can do better
because we can replace probability functions with probability labellings.
Can we generalize this idea efficiently to complex formulas?
The answer is probably negative for the following reasons. While 2SAT, the restriction of the propositional satisfiability problem
to clauses with at most two literals, can be solved in polynomial time,
the probabilistic variant 2PSAT is NP-hard \cite{georgakopoulos1988}. 
If there was a way to connect probability functions over non-atomic arguments to generalized
labellings, it could be used to solve 2PSAT. 
We will explain this in more detail in the remainder of this section.
We will also give some intuition for why complex formulas are difficult to handle.
The results developed here will not be used in subsequent sections. The reader who is not interested in the details can find the main result at the end of this section or can skip the rest of this section
entirely. 

To begin with, we generalize probability labellings to complex formulas.
Instead of assigning probabilities to single arguments, we will assign probabilities
to valuations of $k$ arguments.
Let $\valuations_k$ denote the set of all $k$-valuations of the form $(A_1 = b_1, \dots,A_k = b_k)$,
where $A_1, \dots, A_k$ are distinct arguments from $\args$ and $b_i \in \{0,1\}$ evaluates 
the $i$-th arguments as accepted ($1$) or rejected ($0$).
That is, every $v \in \valuations_k$ evaluates $k$ arguments as either accepted or rejected.
How many $k$-valuations are there? We can pick $\binom{n}{k}$ $k$-elementary subsets of arguments
and for each such subset, there are $2^k$ possible valuations. Hence, the number of $k$-valuations
is $2^k \cdot \binom{|\args|}{k}$, which is polynomial in the number of arguments. 

We define a $k$-th order 
labelling as a function $\lambda: \valuations^k \rightarrow [0,1]$ 
that assign probabilities to $k$-tuples of arguments.
We can use $k$-th order labellings to assign probabilities to formulas over $k$ 
arguments.
We say that a valuation $(A_1 = b_1, \dots,A_k = b_k) \in \valuations_k$ satisfies an argument $A \in \args$ iff $A=A_i$ for some $i \in \{1,\dots, k\}$ and $b_i=1$, that is, if it evaluates $A$ and evaluates $A$
as accepted. A valuation satisfies a complex formula if it evaluates all arguments in the complex formula
and evaluates the formula to true when interpreting it in the usual recursive way.
Note, that a valuation can satisfy a formula only if all arguments in the formula are evaluated by
the valuation. In particular, a valuation in $\valuations_k$ can satisfy only formulas with at most 
$k$ arguments by definition.

We define the probability of a formula $F$ under a $k$-th order labelling $\lambda$ 
as $\lambda(F) = \sum_{v \in \valuations_k, v \models F} \lambda(w)$ as before.
A $k$-th order labelling $\lambda$ satisfies a linear atomic constraint iff
$\sum_{i=1}^n c_i \cdot \lambda(F_i) \leq c_0$.
As before, we would like to summarize probability functions in equivalence classes such that
every labelling corresponds to the probability functions in one equivalence class.
To begin with, we define what we mean by correspondence.
\begin{definition}
A $k$-th order labelling $\lambda$ corresponds to a probability function $P$ if 
$\lambda(F) = P(F)$ for all formulas $F$ that contain at most $k$ arguments. 
\end{definition}
Clearly, every valuation of $|\args|$ arguments corresponds to a possible world
and so every $|\args|$-order labelling corresponds to a probability function.
However, this is not an interesting relationship because
$|\args|$-order labellings are still exponentially large.
Indeed, as we explained before, there are $2^k \cdot \binom{|\args|}{k}$ possible
$k$-valuations, so the size of $k$-th order labellings is exponential in $k$.
However, if $k$ is small, this is not a problem.

In order to connect $k$-th order labellings to probability functions, we can generalize
atomic equivalence to $k$-th order equivalence. Formally, we call two 
probability functions $P_1, P_2$ $k$-th order equivalent and write $P_1 \equiv_k P_2$ if 
$P_1(\bigwedge_{i=1}^k A_i^{b_i}) = P_2(\bigwedge_{i=1}^k A_i^{b_i})$
for all distinct arguments  $A_1, \dots, A_k \in \args$ and $b_i \in \{0,1\}$, 
where we let $A_i^0 := \neg A_i$ and $A_i^1 := A_i$.
Notice that first-order equivalence is basically atomic equivalence.
 Roughly speaking, two probability functions are $k$-th order equivalent if they assign the same probabilities to all conjunctions that contain exactly $k$ distinct arguments
in positive or negative form.
In order to apply our previous idea, we need a bijective mapping between $k$-th order equivalence classes
of probability functions and $k$-th order labellings.
Unfortunately, not every $k$-th order labelling corresponds to
a probability function.
\begin{example}
The definition of a first-order labelling 
admits the labelling $\lambda$ with $\lambda(A=0) = \lambda(A=1) = 1$. 
That is, $\lambda(\neg A) = \lambda(A) = 1$
However, for every probability function $P$, we necessarily have 
$P(\neg A) = 1- P(A)$. Hence, $\lambda$ cannot correspond to a probability function.
\end{example}
In general, in order to establish an interesting relationship
between $k$-th order labellings and probability functions, we have to add additional 
conditions that correspond to normalization and additivity of probability functions.
How many of these conditions do we need? 
For $k=1$, we only need one normalization condition $\lambda(A) = 1 - \lambda(\neg A)$ for every $A \in \args$. Hence.
$|\args|$ conditions are sufficient. 
%Note, in particular, that every probabilistic labelling $L$ considered before corresponds to a 
%first-oder labelling $\lambda$ defined by $\lambda(A=1) = L(A)$ and  $\lambda(A=0) = 1 - L(A)$ for all $A \in \args$.
Does the number keep polynomial for $k>1$.
First note that normalization is not sufficient for $k>1$. We also have to assure some marginal 
consistency conditions
that assure that the probabilities of overlapping valuations are consistent.
We know from item 1 in Lemma \ref{lemma_computation_rules} that probability functions satisfy $P(A_1 \wedge \neg  A_2) + P(A_1 \wedge A_2) =
P(A_1) = P(A_1 \wedge \neg A_3) + P(A_1 \wedge A_3)$
for all $A_1, A_2, A_3 \in \args$.
Therefore, for $k=2$, we need at least the following two types of conditions:
\begin{enumerate}
	\item $\lambda(\neg A_1 \wedge \neg  A_2) + \lambda(\neg A_1 \wedge A_2) + \lambda(A_1 \wedge \neg A_2) + \lambda(A_1 \wedge A_2) = 1$
for all distinct $A_1, A_2 \in \args$ (normalization) and
	\item $\lambda(A_1 \wedge \neg  A_2) + \lambda(A_1 \wedge A_2) = \lambda(A_1 \wedge \neg A_3) + \lambda(A_1 \wedge A_3)$
for all distinct $A_1, A_2, A_3 \in \args$ (marginal consistency).
\end{enumerate}
The number of normalization and marginal consistency conditions is still polynomial.
Unfortunately, they are not sufficient to guarantee that every labelling 
corresponds to a probability function.
\begin{example}
\label{example_3_order_labelling_constraints}
Consider three arguments $A, B, C$ and the second-order labelling $\lambda$ 
shown in Table \ref{example_3_order_labelling_constraints_figure}. 
It is easy to check that $\lambda$ does indeed satisfy
the normalization and marginalization conditions. Now assume that $\lambda$
corresponds to a probability function $P$. 
Since a conjunction entails every subconjunction, we get from item 2 of
Lemma \ref{lemma_computation_rules}  that
$P(A \wedge B  \wedge C) \leq P(A \wedge C) = \lambda(A \wedge C) = 0$
and that
$P(A \wedge B  \wedge \neg C) \leq P(B \wedge \neg C) = \lambda(B \wedge \neg C) = 0$.
From item 1 of Lemma \ref{lemma_computation_rules}, we get
$P(A \wedge B )  = P(A \wedge B  \wedge C) + P(A \wedge B  \wedge \neg C)$.
This implies $0.5 = \lambda(A \wedge B ) = P(A \wedge B ) = P(A \wedge B  \wedge \neg C) + P(A \wedge B  \wedge \neg C)
 = 0 +0 =0$, which is clearly a contradiction. Hence, $\lambda$ does not correspond
to a probability function even though it satisfies our basic normalization and marginalization conditions.
\begin{table}
	\begin{tabular}{lllllllll}
		\hline
		$A$ & $B$ & $\lambda(A, B)$ & $A$ & $C$ & $\lambda(A, C)$ & $B$ & $C$ & $\lambda(B,C)$   \\[0.0cm]
		\hline
		$0$ & $0$ & $0.5$ & $0$ & $0$ & $0$ & $0$ & $0$ & $0.5$  \\[0.0cm]
		$0$ & $1$ & $0$ & $0$ & $1$ & $0.5$ & $0$ & $1$ & $0$  \\[0.0cm]
		$1$ & $0$ & $0$ & $1$ & $0$ & $0.5$ & $1$ & $0$ & $0$  \\[0.0cm]
		$1$ & $1$ & $0.5$ & $1$ & $1$ & $0$ & $1$ & $1$ & $0.5$  \\[0.0cm]
	\end{tabular}
	\caption{3rd-order labelling for Example \ref{example_3_order_labelling_constraints}
	\label{example_3_order_labelling_constraints_figure}}	
\end{table}
\end{example}
Example \ref{example_3_order_labelling_constraints} shows that it does not suffice to guarantee
marginalization consistency for up to $k$ arguments. Even though the probabilities may be locally
consistent, they may be impossible to obtain from a probability function over all arguments.  

Is there a clever way to extend our conditions? In order to shed some light on this question,
let us first note that all our conditions can be expressed
as linear constraints.
For example, for $k=2$, we have the constraints
\begin{enumerate}
	\item $\pi(\neg A_1 \wedge \neg  A_2) + \pi(\neg A_1 \wedge A_2) + \pi(A_1 \wedge \neg A_2) + \pi(A_1 \wedge A_2) = 1$
for all distinct $A_1, A_2 \in \args$ (normalization) and
	\item $\pi(A_1 \wedge \neg  A_2) + \pi(A_1 \wedge A_2) = \pi(A_1 \wedge \neg A_3) + \pi(A_1 \wedge A_3)$
for all distinct $A_1, A_2, A_3 \in \args$ (marginal consistency).
\end{enumerate}
Clearly, a $k$-th order labelling satisfies the conditions iff it satisfies the linear constraints.
Indeed, it is hard to imagine how non-linear conditions could help to solve the problem in Example \ref{example_3_order_labelling_constraints}. Therefore, we may as well ask, is there a polynomially sized set of linear constraints such that every $k$-th order labelling that satisfies the constraints corresponds to a probability function? As we saw before, the answer is yes for $k=1$ because we only need one normalization
condition for every argument. 
However, for $k>1$, such a set cannot exist under the usual complexity-theoretical assumptions.
\begin{proposition}
If $P \neq NP$, then, for $k>1$, there is no set of linear constraints $C$ such that the size of $C$ is polynomial in $\args$ and every $k$-th order labelling that satisfies $C$ corresponds to a probability function.
\end{proposition}
\begin{proof}
We show the contrapositive by showing that under the assumption that $C$ has polynomial size, 
$2PSAT$ can be solved in polynomial time. 
Since 2SAT is NP-hard \cite{georgakopoulos1988}, this implies $P = NP$.

For the sake of contradiction, assume that such a set $C$ exists. 
A $2PSAT$ instance consists of $n$ propositional atoms $\alpha$
and $m$ propositional 2-clauses of the form 
$\pi(\alpha_{i,1}^{b_{i,1}} \vee \alpha_{i,2}^{b_{i,1}}) = p$.
The problem is to decide whether the clauses can be satisfied by
a probability function \cite{georgakopoulos1988}.
Clearly, we can introduce an argument $A$ for every propositional atom $\alpha$ and 
represent every statement $\pi(\alpha_{i,1}^{b_{i,1}} \vee \alpha_{i,2}^{b_{i,2}}) = p_i$
with a linear constraint $\pi(A_{i,1}^{b_{i,1}} \vee A_{i,2}^{b_{i,2}}) = p_i$.
Then the set of linear constraints can be satisfied if and
only if the 2PSAT instance can be satisfied. 

In order to decide satisfiability, we can build up a linear optimization problem 
over $k$-th order labellings similar to the proof of Proposition \ref{prop_sat_polynomial_time}.
The number of optimization variables is $2^k \cdot \binom{|\args|}{k}$ (number of probabilities
stored in the probability labelling) plus $2m$ (two slack variables for every clause).
The probability and slack variables are again non-negative.
We add $m$ constraints $\pi(A_{i,1}^{b_{i,1}} \vee A_{i,2}^{b_{i,2}}) = p_i + s^+_i - s^-_i$.
The term $(s^+_i - s^-_i)$ guarantees that the constraint can be satisfied. Note that the original
constraint is satisfied iff $s^+_i = s^-_i = 0$ (we will again minimize the slack variables). 
Furthermore, we need the constraints from $C$. The objective function is
$\sum_{i=1}^m (s^+_i + s^-_i)$. Since $C$ is supposed to be of polynomial size,
the linear optimization problem has polynomial size. Therefore, it can be solved in polynomial time.

Similar to the proof of Proposition \ref{prop_sat_polynomial_time}, we can check that the
clauses are satisfiable by a $k$-th order labelling if and only if the linear optimization
problem has minimum $0$. If the clauses are satisfiable, consider a probability function $P$
that satisfies the constraints. We can obtain a corresponding $k$-th order labelling $\lambda_P$ by assigning
probability $P(\bigwedge_{k=1}^m A_i^{b_i})$ to the assignment $(A_1 = b_1, \dots,A_k = b_k) \in \valuations_k$. Since $P$ satisfies the constraints, $\lambda_P$ clearly satisfies them as well.
Therefore, we can let $s^+_i = s^-_i = 0$  for $i=1,\dots,m$ and the minimum is $0$.

Conversely, if the minimum is $0$, there is a $k$-th order labelling $\lambda$ that satisfies
$\lambda(A_{i,1}^{b_{i,1}} \vee A_{i,2}^{b_{i,2}}) = p_i + 0 + 0 = p_i$  for $i=1,\dots,m$.
That is, it satisfies all clauses. It also satisfies $C$ and, therefore, corresponds to a probability
function $P$. Hence, $P$ satisfies the clauses and so the clauses are satisfiable. 

Hence, under the assumption that $C$ has polynomial size, $2PSAT$ can be solved in polynomial time and therefore $P = NP$.
\end{proof}
Hence, assuming $P \neq NP$, our previous idea cannot be generalized to arbitrary non-atomic formulas. 
However, there may still be interesting special cases that can be solved efficiently in other ways. 
We will look at some other fragments in the following sections.

\section{Complex Constraints}
\label{sec_constraints}

In this section, we look at how far we can extend the expressiveness of our constraint language
when keeping the query language atomic.
Unfortunately, there are strong limitations.
Even if we only allow constraining the probability of the disjunction of two literals,
the satisfiability problem becomes intractable.
To show this,
we define a \emph{linear 2DN constraint} as an expression of the form 
$\sum_{i=1}^n c_i \cdot \pi(A_{i,1}^{b_{i,1}} \vee A_{i,2}^{b_{i,2}}) \leq c_0$,
where $c_i \in \mathbb{R}$, $A_{i,j} \in \args$ and $b_{i,j} \in \{0,1\}$ 
and we let $A_{i,j}^0 := \neg A_{i,j}$ and $A_{i,j}^1 := A_{i,j}$.
As before, we say that a probability function $P$ satisfies such a constraint iff
$\sum_{i=1}^n c_i \cdot P(A_{i,1}^{b_{i,1}} \vee A_{i,2}^{b_{i,1}}) \leq c_0$.
\begin{proposition}
\label{prop_sat_2DN_constraints}
The satisfiability problem for Linear 2DN Constraints is NP-complete.
\end{proposition}
\begin{proof}
For membership, we need a result from Linear Programming theory.
Among the optimal solutions of an N-dimensional linear program, there must be one
that satisfies $N$ constraints with equality \cite{bertsimas97}.
Here, our constraints consist of $|\constraints|$ semantical constraints from $\constraints$,
a normalization constraint that makes sure that the probabilities sum to $1$
and $2^n$ non-negativity constraints for the probabilities of possible worlds.
Therefore, if $\constraints$ is satisfiable, there must be an optimal solution $P^*$ that satisfies 
at least $2^n - (|\constraints| + 1)$ non-negativity constraints 
with equality. 
That is, $P^*$ assigns non-zero probability to at most $|\constraints| + 1$ possible worlds.
Let $W = \{w \mid w \in 2^\args, P^*(w) > 0\}$.
Then $|W| \leq |\constraints| + 1$, that is,
$W$ has polynomial size. 
Furthermore, for arbitrary formulas $F$ over $\args$, 
$P^*(F) = \sum_{w \in 2^\args, w \models F} P^*(w) = \sum_{w \in W, w \models F} P^*(w)$.
Hence, the set of pairs $\{(w, P^*(w)) \mid w \in W\}$
provides a certificate of polynomial size such that 
checking the constraints and $P^*(Q)>0$ can be done in polynomial
time. 

For hardness, we can give a polynomial-time reduction from 2PSAT,
the problem of deciding whether a set of probability statements of the form 
$\pi(\alpha_{i,1}^{b_{i,1}} \vee \alpha_{i,2}^{b_{i,1}}) = p$ over propositional 2-clauses
is satisfiable. As shown in \cite{georgakopoulos1988}, 2PSAT is NP-complete.
We can introduce an argument $A$ for every propositional atom $\alpha$ 
and represent every statement $\pi(\alpha_{i,1}^{b_{i,1}} \vee \alpha_{i,2}^{b_{i,1}}) = p$
with a linear 2D constraint $\pi(A_{i,1}^{b_{i,1}} \vee A_{i,2}^{b_{i,1}}) = p$.
Then, clearly, the set of linear 2D constraints can be satisfied if and
only if the 2PSAT instance can be satisfied.
\end{proof}
The problem does not get significantly easier when considering conjunction instead of disjunction.
We define a \emph{linear 2CN constraint} as an expression of the form 
$\sum_{i=1}^n c_i \cdot \pi(A_{i,1}^{b_{i,1}} \wedge A_{i,2}^{b_{i,2}}) \leq c_0$
and say that $P$ satisfies such a constraint iff
$\sum_{i=1}^n c_i \cdot P(A_{i,1}^{b_{i,1}} \wedge A_{i,2}^{b_{i,1}}) \leq c_0$.
\begin{proposition}
The satisfiability problem for Linear 2CN Constraints is NP-complete.
\end{proposition}
\begin{proof}
Membership follows as in the previous proposition.

For hardness, we can give a polynomial-time reduction from satisfiability
of linear 2DN constraints that we considered before.
Consider an arbitrary linear 2DN constraint $\sum_{i=1}^n c_i \cdot \pi(A_{i,1}^{b_{i,1}} \vee A_{i,2}^{b_{i,2}}) \leq c_0$.
Notice that every formula $A_{i,1}^{b_{i,1}} \vee A_{i,2}^{b_{i,2}}$ can be equivalently expressed as a disjunction of three exclusive conjunctions of length $2$.
For example $A \vee \neg B \equiv (A \wedge \neg B) \vee (\neg A \wedge \neg B) \vee (A \wedge B)$.
Since these conjunctions cannot be satisfied simultaneously, 
we have $P\big((A \wedge \neg B) \vee (\neg A \wedge \neg B) \vee (A \wedge B)\big)
= P(A \wedge \neg B) + P(\neg A \wedge \neg B) + P(A \wedge B)$.
In general, every linear 2DN constraint $\sum_{i=1}^n c_i \cdot \pi(A_{i,1}^{b_{i,1}} \vee A_{i,2}^{b_{i,2}}) \leq c_0$ can be equivalently represented by
a linear 2CN constraint
$\sum_{i=1}^n c_i \cdot \big(\sum_{i=1}^3 \pi(C_i) \big) ) \leq c_0$
where the $C_i$ are exclusive conjunctions of two literals chosen as before to satisfy
$(A_{i,1}^{b_{i,1}} \vee A_{i,2}^{b_{i,2}}) \equiv \bigvee_{i=1}^3 C_i$.
The number of constraints remains unchanged and their size changes only by a constant factor.
In particular, a set of linear 2DN constraints is satisfiable if and only if the corresponding
set of linear 2CN constraints is satisfiable.
\end{proof}
So talking about the probability of formulas in constraints is inherently difficult.
However, instead of allowing logical connectives in probability statements, we could 
consider logical connections of constraints as considered in \cite{HunterPT2018Arxiv}.
Note that connecting constraints conjunctively does not add anything semantically.
This is because there is no difference between adding two constraints or their
conjunction to a knowledge base when the usual interpretation of conjunction is used.
Adding negation basically means allowing for strict inequalities. Negation alone does not
add any additional difficulties and the problem can be reduced to the case without negation
with constant cost \cite{HunterPolbergPotyka18}.
The most interesting case is allowing for connecting constraints disjunctively.
We define a \emph{2D linear atomic constraint} as an expression of the form 
$l_1 \vee l_2$,
where $l_1, l_2$ are linear atomic constraints.
We say that a probability function $P$ satisfies such a constraint iff
it satisfies $l_1$ or $l_2$.
Unfortunately, the satisfiability problem for 2D linear atomic constraints is again NP-hard.
\begin{proposition}
The satisfiability problem for 2D Linear Atomic Constraints is NP-complete.

In particular, the problem remains NP-hard 
even when the constraints are further restricted to the form
$\big(\sum_{i=1}^2 c_i \cdot \pi(A_i) \leq c_0\big) \vee \big(\sum_{i=1}^2 c_i \cdot \pi(B_i) \leq c_0\big)$, that is, even when the linear atomic constraints in the disjunction can contain at most $2$ probability terms.
\end{proposition}
\begin{proof}
Membership follows again from noticing that a labelling that satisfies the constraints is a certificate
that can be checked in polynomial time.

For hardness, we give a polynomial-time reduction from 3SAT
to satisfiability of 2D Linear Atomic Constraints.
As before, for every propositional atom, we introduce a corresponding argument.
Consider a clause $(\alpha^{b_1}_1 \vee \alpha^{b_2}_2 \vee \alpha^{b_3}_3)$.
We introduce three additional auxiliary arguments $X_1, X_2, X_3$ and encode the clause by four 
2D linear atomic constraints. We use the constraints 
$\big(\pi(A_i) = b_i\big) \vee \big(\pi(X_i)=1\big)$ for $i=1,\dots,3$ 
and $\big(\pi(X_1) + \pi(X_2) \leq 1\big) \vee \big(\pi(X_2) + \pi(X_3) \leq 1\big)$.	
Notice that $\pi(A) = 1$ is equivalent to $-\pi(A) \leq -1$ and $\pi(A) = 0$ is equivalent to $\pi(A) \leq 0$.
The constraint $\pi(X_i)=1$ must be satisfied if the $i$-th literal is not satisfied. 
The last constraint expresses
that at most two literals are allowed to be falsified. If all three literals are falsified, the last
constraint is not satisfied. If the first or second literal are satisfied, the first atom in the disjunction will be satisfied, if the third literal is satisfied, the second atom will be satisfied. 
Our reduction introduces $3m$ new arguments and $3m$ additional constraints, 
so the size is polynomial.

If $F$ is satisfiable, then there is a possible world $w$ (interpretation) that satisfies $F$.
Consider the probability function $P_w$ that assigns probability $1$ to $w$ and $0$ to all other worlds.
Let $L$ denote the probability labelling corresponding to $P_w$.
We extend $L$ to a probability labelling $L'$ over the $n+3m$ arguments.
For every clause $(A^{b_1}_1 \vee A^{b_2}_2 \vee A^{b_3}_3)$, $w$ satisfies one literal $A^{b_i}_i$
and we set the corresponding auxilary argument $X_i$ to $0$ and the other two to $1$. 
Then the 2D linear atomic constraints are satisfied by $L'$ and the corresponding probability
function $P_{L'}$ satisfies the constraint as well as shown before. Hence, the 
2D linear atomic constraints are satisfiable.

Conversely, if all 2D linear atomic constraints are satisfied by a probability function $P$,
then every world $w$ with $P(w) > 0$ must satisfy $F$ (strictly speaking, $w$ also interprets the auxiliary
arguments, but those can just be ignored). For the sake of contradiction, 
assume that this is not the case. That is, there is a world $w$ with $P(w) > 0$ that does not satisfy $F$.
Then there is a clause $(A^{b_1}_1 \vee A^{b_2}_2 \vee A^{b_3}_3)$ in $F$ such that $w$ 
satisfies neither $A^{b_1}_1$ nor $A^{b_2}_2$ nor $A^{b_3}_3$.
If $b_i=1$, then $P(A_i) = \sum_{v \in \args, A_i \in v} P(v) \leq  \sum_{v \in \args\setminus w} P(v)
= 1 - P(w) < 1$ and if $b_i=0$, then $P(A_i) \geq P(w) > 0$. 
Then the constraints $\big(\pi(A_i) = b_i\big) \vee \big(\pi(X_i)=1\big)$ can only be satisfied if
$P(X_i)=1$ for $i=1,\dots,3$. But then $P(X_1) + P(X_2) =  P(X_2) + P(X_3) = 2$
and the constraint 
$\big(\pi(X_1) + \pi(X_2) \leq 1\big) \vee \big(\pi(X_2) + \pi(X_3) \leq 1\big)$
is violated,
which contradicts our assumption that $P$ satisfies the constraints.
Hence, indeed every world $w$ with $P(w) > 0$ must satisfy $F$ and since there must be at least one world
with non-zero probability (otherwise, $P$ cannot be a probability function), $F$ is satisfiable.
\end{proof}
It seems that our constraint language cannot be extended significantly without loosing our polynomial
runtime guarantees.
However, there is still one interesting special case left that we did not exclude so far.
It may be possible to extend our fragment to statements of the form 
$\big(\pi(A_1) \leq c_1\big) \vee \big(\pi(A_2) \leq c_2\big)$, where every linear atomic constraint can only 
contain a single probability term. 
This would allow making conditional statements like in the rationality property \cite{hunter2014probabilistic}:
\begin{description}
 \item[RAT:] $P$ is called \emph{rational} if for all $A, B \in \args$ with $(A, B) \in \attacks$, we have $P(A) > 0.5$ implies $P(B) \leq 0.5$.
\end{description}
We may reuse ideas for 2SAT in order to handle such constraints efficiently.
However, we currently cannot say for certain if this is possible in polynomial time and leave this question for future work.

\section{Complex Queries}
\label{sec_queries}

As we saw in the previous section, our constraint language cannot be extended significantly without
loosing polynomial runtime guarantees.
We will now conduct a similar analysis for the query language.
That is, we investigate how far we can extend the expressiveness of our query language
when keeping the constraint language atomic.

Unfortunately,
there is again no efficient way to answer arbitrarily complex queries
because, otherwise, the entailment problem could be used
to solve the propositional satisfiability problem.
In order to make this precise, we define a \emph{3CNF-Query} as a  formula 
$Q = \bigwedge_{i=1}^m \big(\bigvee_{j=1}^3 A_{i,j}^{b_{i,j}}\big)$ over arguments, where again 
$b_{i,j} \in \{0,1\}$, $A_{i,j}^0 := \neg A_i$ and $A_{i,j}^1 = A_i$.
Just deciding whether the upper probability bound for a 3CNF-Query is non-zero is NP-hard already. 
\begin{proposition}
\label{prop_ent_general_queries}
Let $\constraints$ be a satisfiable set of linear atomic constraints over $\args = \{A_1,\dots, A_n\}$ 
and let $Q$ be a 3CNF-query.
Then the following problem is NP-complete: decide whether the upper bound on the probability of $Q$ among the probability functions that satisfy $C$ is non-zero.
\end{proposition}
\begin{proof}
For membership, we can construct a polynomial certificate like in the proof of Proposition \ref{prop_sat_2DN_constraints}.

For hardness, we give a polynomial-time reduction from $3SAT$.
Given a propositional 3CNF formula $F$ with $n$ atoms $\alpha_i$,
we introduce corresponding arguments $A_i$.
Let $Q$ be the query obtained from $F$ by replacing $\alpha_i$ with $A_i$ for $i=1,\dots,n$.
We do not add any constraints, so that all $P \in \probDists$ satisfy our constraints trivially.
Then the upper bound on the probability of $Q$ is non-zero iff $F$ is satisfiable. To see this, note that
if $F$ is satisfiable, there is an interpretation that satisfies $F$ and a corresponding 
possible world $w$ that satisfies $Q$. Then the probability function $P_w$ with 
$P_w(w) = 1$ and $P_(w') = 0$ for all other possible worlds gives 
$P_w(Q) = P_w(w) = 1 > 0$.
Conversely, if $F$ is not satisfiable, $Q$ is not satisfiable either and $P(Q) = \sum_{w \in 2^\args, w \models Q} P(w) = 0$ because the sum does not contain any terms for any $P \in \probDists$.
\end{proof}
However, there may be some interesting special cases that can be solved efficiently
if we make additional assumptions.
One case is answering conjunctive queries under the principle of maximum entropy
as we explain in the following. However, the derived probabilities have to be considered
with care as we explain at the end of this section.

\subsection{Answering Conjunctive Queries under the Principle of Maximum Entropy}

When reasoning under the principle of maximum entropy,
we do not consider all probability functions that satisfy our constraints anymore,
but restrict to the one that maximizes entropy \cite{JohnsonShore83,Jaynes83,ParisVencovska90}.
The entropy of a probability function $P$ over $\args$ is defined as
$H(P) = - \sum_{w \in 2^\args} P(w) \cdot \log P(w)$, where $0 \cdot \log 0$ is defined as $0$.
The entropy can be seen as a measure of uncertainty. Indeed, the entropy is always non-negative 
and maximal if $P$ is the uniform distribution.
Therefore, if no constraints are given (no information about the arguments), the principle of maximum entropy is roughly equivalent to the principle of indifference. When accepting the uniform distribution
as the simplest probabilistic model, the principle of maximum entropy can also be seen as a probabilistic
version of Occam's razor.
Intuitively, by maximizing entropy among the probability  functions that satisfy a set of constraints,
we select the probability distribution that adds as little information as possible.
In addition to these intuitive justifications, the principle of optimum entropy has been justified by several characterizations with common-sense properties \cite{JohnsonShore83,Jaynes83,ParisVencovska90,Kern-Isberner00d}.

For a probability labelling $L$ over $n$ arguments $\args = \{A_1\dots, A_n\}$, 
we define its entropy as $H(L) = \sum_{i=1}^n \big(- L(A_i) \cdot \log L(A_i) - (1 - L(A_i)) \cdot \log (1 - L(A_i)) \big)$. 
In the proof of Lemma \ref{lemma_one_to_one_correspondence}, we identified a very special
probability function $P_L$ in the equivalence class corresponding to $L$.
We highlight this probability function in the following corollary.
\begin{corollary}
For every labelling $L \in \probLabs$, there is a probability function $P_L$ 
such that $P_L(w) = \prod_{A \in w} L(A) \cdot \prod_{A \in \args \setminus w} (1- L(A))$
and $r([P_L]) = L$.
\end{corollary}
As we show next, the entropy $H(L)$ of a labelling $L$ corresponds to the maximum entropy taken in the equivalence class $[P_L]$. In particular, the maximum in $[P_L]$ is always taken by the corresponding probability function $P_L$.
\begin{proposition}
\label{prop_max_entropy_in_equivalence_class}
For every labelling $L \in \probLabs$, the corresponding probability function $P_L$ 
maximizes entropy among all $P' \in [P_L]$.
In particular, $H(P_L) = H(L)$.
\end{proposition}
\begin{proof}
Consider an arbitrary probability function $P \in [P_L]$
For all formulas $F$, we let $\indicator_F: 2^\args \rightarrow \{0,1\}$ denote the indicator function that yields $1$ iff $w \models F$.
Then we have
\begin{align*}
&H(L) - H(P) \\
&= -\sum_{i=1}^n \big(L(A_i) \cdot \log L(A_i) + (1 - L(A_i)) \cdot \log (1 - L(A_i)) \big) \\
&\quad + \sum_{w \in 2^\args} P(w) \cdot \log P(w) \\
&= -\sum_{i=1}^n \big(P(A_i) \cdot \log L(A_i)  + (1 - P(A_i)) \cdot \log (1 - L(A_i)) \big) \\
&\quad + \sum_{w \in 2^\args} P(w) \cdot \log P(w) \\
&= -\sum_{i=1}^n \big(\sum_{w \in 2^\args} \indicator_{A_i}(w) \cdot P(w) \cdot \log L(A_i) \\
&\quad + \sum_{w \in 2^\args} \indicator_{\neg A_i}(w) \cdot P(w)  \cdot \log (1 - L(A_i)) \big) \\
&\quad + \sum_{w \in 2^\args} P(w) \cdot \log P(w) \\
&= -\sum_{w \in 2^\args} P(w) \sum_{i=1}^n \big( \indicator_{A_i}(w) \log L(A_i) 
+ \indicator_{\neg A_i}(w) \log (1 - L(A_i)) \big) \\
&\quad + \sum_{w \in 2^\args} P(w) \cdot \log P(w) \\
&= -\sum_{w \in 2^\args} P(w) \log \big(\prod_{A_i \in w} L(A_i) \cdot \prod_{A_i \in \args \setminus w} (1- L(A_i)) \big) \\
&\quad + \sum_{w \in 2^\args} P(w) \cdot \log P(w) \\
&= \sum_{w \in 2^\args} P(w) \log \frac{P(w)}{\prod_{A_i \in w} L(A_i) \cdot \prod_{A_i \in \args \setminus w} (1- L(A_i))} \\
&= \sum_{w \in 2^\args} P(w) \log \frac{P(w)}{P_L(w)} = KL(P,P_L) \geq 0,
\end{align*}
where, for the second equality, we used the fact that $P(A) = L(A)$ for all $A \in \args$ and in the last row,
we used the observation that the previous formula corresponds to the KL-divergence between two probability functions that is always non-negative.
Furthermore, the KL-divergence is $0$ if and only if both arguments are equal \cite{yeung2008information}, that is, $KL(P,P_L)=0$ if and only if $P=P_L$. 
Therefore, $H(L) = H(P_L)$ and $H(L) > H(P)$ whenever $P \neq P_L$.
In particular, $H(P_L) > H(P)$ for all  $P \in [P_L] \setminus \{P_L\}$.
\end{proof}
Hence, in order to compute the probability function with maximum entropy, we can just compute the 
labelling $L^*$ with maximum entropy. The corresponding probability function $P_{L^*}$ then maximizes entropy.
This is the basic idea of the following proposition.
\begin{proposition}
\label{prop_computing_ME_dist}
Given a satisfiable finite set of linear atomic constraints $\constraints$,
the optimization problem 
\begin{align*}
\arg \max_{P \in \probDists} \quad &H(P) \\
\textit{such that} \quad &P \models \constraints.
\end{align*}
has a unique solution $P^*$ and $L_{P^*}$
is the unique solution of the optimization problem
\begin{align*}
\arg \max_{L \in \probLabs} \quad &H(L) \\
\textit{such that} \quad &L \models \constraints.
\end{align*}
In particular, $L_{P^*}$ can be computed in polynomial time.
\end{proposition}
\begin{proof}
Both optimization problems have a strictly concave and continuous objective function.
Maximizing such a function subject to consistent linear constraints yields a unique solution \cite{Nocedal2006}.
In particular, these problems can be solved by interior-point methods in polynomial time in the number
of optimization variables and constraints \cite{Boyd2004}. For the first problem, the number of optimization
variables is exponential in the number of arguments, but for the second problem the number of optimization
variables equals the number of arguments. Hence, the problem can be solved in polynomial time. 
Since the solution $L^*$ of the second problem maximizes entropy among all probability labellings, and the probability distributions
corresponding to the labellings maximize entropy among their equivalence classes according to Proposition \ref{prop_max_entropy_in_equivalence_class}, $L^*$ must equal $L_{P^*}$.
\end{proof}
Having computed $L_{P^*}$, we can compute a compact representation of $P^*$. Of course, constructing 
$P^*$ explicitly would take exponential time again. Fortunately, for some queries, we can just work 
with the compact representation directly. This includes, in particular, conjunctive queries, as we explain
in the following proposition.
\begin{proposition}
\label{prop_max_ent_conj_queries}
Let $\args = \{A_1,\dots, A_n\}$, let $\constraints$ be a satisfiable set of linear atomic constraints and let $Q$ be a conjunction of literals, that is, $Q = \bigwedge_{i \in I} A_i^{b_i}$, where
$I \subseteq \{1,\dots,n\}$, $b_i \in \{0,1\}$.
Let $P^*$ be the probability function that maximizes entropy among all probability functions
that satisfy $\constraints$.
Then $P^*(Q)$ can be computed in polynomial time even if $P^*$ is unknown. 
In particular, 
\begin{align*}
P^*(\bigwedge_{i \in I} A_i^{b_i}) = \prod_{i \in I} L^*(A_i)^{b_i} \cdot (1 - L^*(A_i))^{1-b_i},
\end{align*}
where $L^*$ is the probability labelling that maximizes entropy among all probability labellings
that satisfy $\constraints$.
\end{proposition}
\begin{proof}
First note that $L(A_i)^{b_i}=L(A_i)$  if $b_i=1$ and $L(A_i)^{b_i}=1$ otherwise.
Dually, $(1 - L(A_i))^{1-b_i} = 1 - L(A_i)$ if $b_i=0$ and $(1 - L(A_i))^{1-b_i}=1$ otherwise.
Let $S = \bigcup_{i \in I} \{A_i\}$ denote the atoms occuring in $Q$.
We know from Proposition \ref{prop_max_entropy_in_equivalence_class} that for all $w \in 2^\args$ with $w \models Q$, we have
\begin{align*}
&P^*(w) 
= \prod_{A \in w} L(A) \cdot \prod_{A \in \args \setminus w} (1- L(A)) \\
&= \big( \prod_{i \in I} L(A_i)^{b_i} (1 - L(A_i))^{1-b_i} \big) \cdot 
\big( \prod_{A \in w \setminus S} L(A) \hspace{-0.4cm} \prod_{A \in \args \setminus \big(S \cup w\big)} \hspace{-0.4cm} (1- L(A)) \big),
\end{align*} 
where we split up the arguments in $S$ (indexed by $I$) since we know their interpretation
(because $w \models Q$).
Therefore,
\begin{align*}
&\quad P^*(Q)
= \sum_{w \in 2^\args, w \models Q} P(w) \\
&=\big( \prod_{i \in I} L(A_i)^{b_i} (1 - L(A_i))^{1-b_i} \big) 
\hspace{-0.2cm} \sum_{w \in 2^{\args\setminus S}} 
\prod_{A \in w} L(A)  
\hspace{-0.3cm} \prod_{A \in \args \setminus (S \cup w)} 
\hspace{-0.5cm} (1- L(A)) \\
&= \prod_{i \in I} L(A_i)^{b_i} \cdot (1 - L(A_i))^{1-b_i},
\end{align*}
where we used the fact that 
$\sum_{w \in 2^{\args\setminus S}} \prod_{A \in w} L(A) \cdot \prod_{A \in \args \setminus (S \cup w)} (1- L(A)) = 1$
as we explained in the proof of Lemma \ref{lemma_one_to_one_correspondence} (the products correspond to probabilities of
a probability function over $\args\setminus S$).

$\prod_{i \in I} L(A_i)^{b_i} \cdot (1 - L(A_i))^{1-b_i}$ can be computed in linear time when we know $L^*$.
We can compute $L^*$ in polynomial time as explained in Proposition \ref{prop_computing_ME_dist}. 
Hence, we can compute $P^*(Q)$ in polynomial time. 
\end{proof}
However, even under the principle of maximum entropy, queries cannot become arbitrarily complex.
In this case, 3CNF-queries are even sufficient to solve $\#3SAT$, the problem of counting the 
interpretations that satisfy a 3CNF formula.
\begin{proposition}
The following problem is $\#P$-hard:
Given a satisfiable set of linear atomic constraints  $\constraints$ over $\args$ 
and a 3CNF-query $Q$, compute $P^*(Q)$, where $P^*$ is the probability function that maximizes entropy among all probability functions
that satisfy $\constraints$.
\end{proposition}
\begin{proof}
We give a polynomial-time reduction from $\#3SAT$.
Given a propositional 3CNF formula $F$, we construct a corresponding 
argument query $Q$ as in the proof of Proposition \ref{prop_ent_general_queries}.
We let $\constraints = \emptyset$ so that $P^*$ is just the uniform distribution with
$P^*(w) = \frac{1}{2^n}$ for all $w \in 2^\args$.
Then $P^*(Q) = \sum_{w \in 2^\args, w \models Q} P^*(w) =  \frac{|\{w \in 2^\args \mid w\models Q\}|}{2^n}$
and $P^*(Q) \cdot 2^n$ is the number of possible worlds that satisfy $Q$, which equals the number
of propositional interpretations that satisfy $F$.
\end{proof}
Similar to the proof of Proposition \ref{prop_ent_general_queries}, it can be seen that the corresponding 
decision problem that asks whether the query has a non-zero probability, is NP-complete. 
While queries can be difficult to compute in general, there are still some special cases that can be solved
efficiently. For example, consider the query $A \vee B$ that asks for the probability that $A$ or $B$ (or both) 
are accepted. Then the query is equivalent to $\big(A \wedge B\big) \vee \big(A \wedge \neg B\big) \vee\big(\neg A \wedge B\big)$.
Since the three conjunctions are exclusive (they cannot be satisfied by the same worlds),
we have $P(A \vee B) = P(A \wedge B) + P(A \wedge \neg B) + P(\neg A \wedge B)$. Hence, we can answer the disjunctive query
by three conjunctive queries that can be computed in polynomial time.
More generally, if we can rewrite a query efficiently as a disjunction of $k$ exclusive conjunctions, 
the query can be answered by $k$ conjunctive queries. However, in general, $k$ can grow exponentially with
the number of atoms in the query. 

\subsection{Independency Assumptions under the Principle of Maximum Entropy}

Applying the principle of maximum entropy allows us to answer conjunctive queries efficiently.
Unfortunately, in some cases, it may give us undesired probabilities.
Readers familiar with the idea of stochastic independence may have noticed that the probability functions
$P_L$ have a very special structure. In the language of probability theory, each $P_L$ assumes stochastic 
independence between all arguments. 
 In the remainder of this section, we will discuss this assumption
and its ramifications.

Formally, stochastic independence means that the probability of
two events happening simultaneously equals the product of the individual probabilities.
Translated to our setting, arguments $A_1$ and $A_2$ are independent if $P(A_1^{b_1} \wedge A_2^{b_2}) = P(A_1^{b_1}) \cdot P(A_2^{b_2})$ for all $b_1, b_2 \in \{0,1\}$,
where again $A_i^0 := \neg A_i$ and $A_i^1 := A_i$.
In general, probabilities are defined by a probabilistic model $P$, for example, by a probability function
or as a probabilistic graphical model like a Bayesian network.
Stochastic independence is therefore a property of a probabilistic model. In our framework,
it is a property of a single probability function. Until now, we computed probabilities
based on all probability functions that satisfy our constraints. While some of these functions satisfy certain independency 
assumptions, others do not. Let us emphasize that the independency assumptions that we talk 
about in this section are not inherent to the epistemic probabilistic argumentation approach in general.
In general, we reason with many probability functions which usually differ in the independency
assumptions that they make.
Not even the probability functions in a single equivalence class must make the same independency assumptions as we illustrate in the following example.
\begin{example}
\label{independency_example}
Consider the BAF $(\{A, B\}, \emptyset, \emptyset)$ and
the labelling $L$ with $L(A) = L(B) = 0.5$.
Figure \ref{fig:independency_example} shows some other probability distributions in $[P_L]$.
We have $P_L(A) = P_1(A) = P_2(A) = 0.5 = P_L(B) = P_1(B) = P_2(B)$ and
therefore $P_L \equiv P_1 \equiv P_2$. $A$ and $B$ are
stochastically independent under $P_L$, but neither under $P_1$ nor under $P_2$.
To see this, note that $P_1(A \wedge B) = 0.5 \neq 0.25 = 0.5 \cdot 0.5 = P_1(A) \cdot P_1(B)$
and $P_2(A \wedge B) = 0 \neq 0.25 = 0.5 \cdot 0.5 = P_2(A) \cdot P_2(B)$.
\begin{table}
	\begin{tabular}{llll}
		\hline
		$w$ & $P_L(w)$ & $P_1(w)$  & $P_2(w)$ \\[0.0cm]
		\hline
		$\emptyset$ & $0.25$ & $0.5$ & $0$ \\[0.0cm]
		$\{A\}$ & $0.25$     & $0$ & $0.5$ \\[0.0cm]
		$\{B\}$ & $0.25$     & $0$ & $0.5$ \\[0.0cm]
		$\{A, B\}$ & $0.25$  & $0.5$ & $0$ \\[0.0cm]
	\end{tabular}
	\caption{Some probability functions in $[P_L]$ for Example \ref{independency_example}.\label{fig:independency_example}}	
\end{table}
\end{example}
The independency assumptions that we talk about here
are enforced only when we answer queries under the principle of maximum entropy and only when 
the constraints are atomic.
They are, in particular, not needed for our previous polynomial-time results. 

Let us now look at the meaning of these independency assumptions.
Recall from Proposition \ref{prop_max_ent_conj_queries} that for all sets of argument indices $I$, we have
\begin{align*}
P^*(\bigwedge_{i \in I} A_i^{b_i}) &= \prod_{i \in I} L^*(A_i)^{b_i} \cdot (1 - L^*(A_i))^{1-b_i} \\
  &=\prod_{i \in I} P^*(A_i^{b_i})
\end{align*}
for the maximum entropy probability function $P^*$.
That is, all arguments are assumed to be stochastically independent under $P^*$. 
Note that this does not mean that $P^*$ cannot capture any
relationship between two arguments. 
In particular, by construction, $P^*$ does maintain all semantical relationships
between arguments that are expressed by the semantical constraints. For example, if we enforce Coherence,
then $P^*$ will satisfy $P^*(B) \leq 1 - P^*(A)$ whenever $A$ attacks $B$.
Note that this is an atomic relationship, whereas 
stochastical independence talks about non-atomic events (at least, a conjunction is involved). 
Since we do not consider any non-atomic constraints in this section, the principle of maximum entropy does indeed imply that all arguments are stochastically independent under $P^*$. However, this does not mean that
$P^*$ would ignore the atomic relationships between arguments that are expressed by our constraints. These relationships are indeed satisfied by definition of $P^*$.
\begin{example}
\label{max_ent_relationships_example}
Consider the BAF $(\{A, B\}, \{(A, B)\}, \emptyset)$ and
the constraint $\pi(A) = 0.8$.
If we also demand Coherence, we have $L^*(A) = 0.8$ and $L^*(B) = 0.2$ for the maximum entropy labelling $L^*$. For the maximum entropy probability function, we have
$P^*(\emptyset) = P^*(\{A,B\}) = 0.8 \cdot 0.2 = 0.16$, 
$P^*(\{A\}) = 0.8^2 = 0.64$ and $P^*(\{B\}) = 0.2^2 = 0.04$.
Furthermore, we have $P^*(A) =  L^*(A)  = 0.8$
and $P^*(B) = L^*(B) = 0.2$.
Hence, $P^*(A\wedge B) = P^*(\{A,B\})  = 0.16 = P(A) \cdot P(B)$, that is,
$A$ and $B$ are stochastically independent under $P^*$.
However, $P^*$ still respects $P(B) \leq 1 - P(A)$ as demanded by coherence. 
\end{example}
However, the assumption of stochastic independence may yield improper probabilities
in some cases. 
Again, if arguments $A$ and $B$ are stochastically independent, then 
$P(A \wedge B) = P(A) \cdot P(B)$.
That is, stochastical independence allows us to compute the probability of a conjunction of arguments by
just multiplying the individual probabilities. What can happen if $P$ assumes stochastical independence
of $A$ and $B$ mistakenly?
To answer this question, it is convenient to introduce some additional tools from probability theory.
For two formulas $F, G$ over $\args$ such that $P(F) > = 0$, the conditional probability of $G$ given $F$ is defined as
$P(G \mid F) = \frac{P(F \wedge G)}{ P(F)}$. Intuitively, $P(G \mid F)$ is the probability of $G$ under the
assumption that $F$ is true.
In the following lemma, we rephrase two basic results from probability theory.
\begin{lemma}
\label{lemma_cond_prob_properties}
For all formulas $F_1, \dots, F_n$ over $\args$, we have
\begin{enumerate}
	\item $P(F_2 \mid F_1) = P(F_2)$ whenever $P(F_1) > 0$ and $F_1$ and $F_2$ are stochastically independent under $P$.
  \item $P(\bigwedge_{i=1}^n F_i) = P(F_1) \cdot \prod_{i=2}^n P(F_i \mid \bigwedge_{k=1}^{i-1} F_k)$.
	\hfill (Chain Rule).
\end{enumerate}
\end{lemma}
\begin{proof}
1. $P(F_2 \mid F_1) = \frac{P(F_2 \wedge F_1}{ P(F_1)} = \frac{P(F_2) \cdot P(F_1)}{ P(F_1)} = P(F_2)$, 
where we first used the definition of conditional probability and then the definition of stochastical
independence.

2. The claim follows by induction. For $n=1$, the claim is trivially true. For the induction step, we have
 \begin{align*}
	P(\bigwedge_{i=1}^{n+1} F_i) 
   &=   P(\bigwedge_{i=1}^{n} F_i) \frac{P(\bigwedge_{i=1}^{n+1} F_i)}{P(\bigwedge_{i=1}^{n} F_i)} \\
	 &=  \bigg( P(F_1) \cdot \prod_{i=2}^n P(F_i \mid \bigwedge_{k=1}^{i-1} F_k) \bigg) \cdot P(F_{n+1} \mid \bigwedge_{k=1}^{n} F_k),  
 \end{align*} 
where we used the induction hypothesis and the definition of conditional probability for the last equality.
\end{proof}
Note that the chain rule does not assume stochastical independence between the formulas.
In particular, for a simple conjunction of two arguments, the chain rule implies 
$P(A \wedge B) = P(A) \cdot P(B \mid A)$.
From item 1 in Lemma \ref{lemma_cond_prob_properties}, we can see that assuming independence of $A$ and $B$ basically means assuming $P(B \mid A) = P(B)$. That is,
knowledge about $A$ does not change our beliefs about $B$. If this assumption is not justified,
the probability of the conjunction may be too low or too large. For example, if $A$ has a positive
impact on the probability of $B$, we have $P(A \wedge B) = P(A) \cdot P(B \mid A) > P(A) \cdot P(B)$,
so that the computed probability is too low when assuming independence of $A$ and $B$.

\subsection{Analysis of Independency Assumptions}

In our framework, arguments should indeed have an impact on other arguments when they are
connected via an attack or support relation.
An attacker should have a negative impact on the probability of an attacked argument
and a supporter should have a positive impact on the probability of a supported argument. 
The independency assumptions make conditional queries meaningless as we explain in the following
corollary.
\begin{corollary}
Let $\constraints$ be a satisfiable set of linear atomic constraints over $\args$ 
and let $P^*$ be the corresponding maximum entropy model.
For all conjunctions $\chi_1, \chi_2$ of arguments from $\args$
such that $P^*(\chi_1) > 0$, we have 
$P^*(\chi_2 \mid \chi_1) = P^*(\chi_2)$.
\end{corollary}
\begin{proof}
Let $\chi_1 = \bigwedge_{i=1}^n A_i$ and $\chi_2 = \bigwedge_{i=1}^m B_i$.
Then $P^*(\chi_1 \wedge \chi_2) = \prod_{i=1}^n P^*(A_i) \cdot \prod_{i=1}^m P^*(B_i)
= P^*(\bigwedge_{i=1}^n A_i) \cdot P^*(\bigwedge_{i=1}^m B_i) = P^*(\chi_1) \cdot P^*(\chi_2)$.
Hence, $\chi_1, \chi_2$ are stochastically independent under $P^*$ and the claim follows from
item 1 of Lemma \ref{lemma_cond_prob_properties}.
\end{proof}
Hence, conditioning on the state of an argument does not have any effect.
However, we may be able to work around this.
In order to simulate conditioning, we can just add constraints that enforce the condition. 
In order to answer a query conditioned on $\bigwedge_{i=1}^n A_i$,
we can add the constraints $\pi(A_i)=1$ for $i=1,\dots,n$, recompute $P^*$ 
and answer a regular conjunctive query.

Since conjunctive queries are related to conditional queries via the chain rule,
the independency assumptions will also affect the probability of conjunctive queries.
However, in this case, the problem can be less severe because the marginal probabilities
may already capture the relationship between connected arguments due to our semantical constraints.
It is difficult to say what probability a conjunctive query should yield in general.
Therefore, we will just look at some simple cases and check whether semantical constraints
can give us plausible guarantees for the maximum entropy model.

To begin with, consider a single attack relation $(A, B) \in \attacks$. 
The most critical case here is that both $A$ and $B$ occur in positive form in a query.
Let us consider a conjunction that contains $A \wedge B$ as a subformula.
By the chain rule, we have for an arbitrary probabilistic model $P$ and an arbitrary conjunction $\chi$ of arguments that
$P(A \wedge B \wedge \chi) = P(A) \cdot P(B \mid A) \cdot P(\chi \mid A \wedge B)$. 
We should probably expect something like $P(B \mid A) \leq 0.5$ (if $A$ is accepted,
B should not be accepted). This implies, in particular, that $P(A \wedge B \wedge \chi) \leq 0.5$,
that is, a conjunction that contains both $A$ and $B$ in positive form should never be accepted.
However, under the maximum entropy model, we have
$P^*(B \mid A) = P^*(B)$. Therefore, the probability of the conjunction can generally 
be too large and may even be greater than $0.5$. 
However, if we employ constraints, then $P^*(B)$ may already contain
the impact of $A$ in a plausible manner. 
Coherence does actually give us an interesting guarantee in this case.
\begin{proposition}
Let $A, B \in \args$ be arguments such that $(A, B) \in \attacks$
and let $\chi$ be a conjunction of arguments from $\args$.
If $P^*$ satisfies coherence (COH), then
$P^*(A \wedge B \wedge \chi) \leq \min \{0.25, 1 - P^*(A)\}$.
\end{proposition}
\begin{proof}
If $P^*$ satisfies coherence, we have $P^*(A \wedge B \wedge \chi) \leq P^*(A \wedge B) = P^*(A) \cdot P^*(B) \leq P^*(A) \cdot (1 - P^*(A))$. From the first-order necessary condition for optimality from differential calculus, we can see that the last term is maximal when $P^*(A) = 0.5$. The maximum is $0.25$, which gives us the first upper bound. We also have
$P^*(A \wedge B \wedge \chi) \leq P^*(A) \cdot (1 - P^*(A)) \leq 1 - P^*(A)$, which gives us the second upper bound.
\end{proof}
We can interpret this as follows: if $P^*(A)$ is close to $1$ (A is close to being classically accepted),
$P^*(A \wedge B \wedge \chi)$ will be close to $0$. This makes sense because accepting $A$ should imply
rejecting $B$ when interpreting attack relations in a classical sense. As $P^*(A)$ moves towards $0$,
the impact of the attack relation becomes gradually weaker. Intuitively, under the maximum entropy model,
the strength of an attack relation is determined by the belief in the source. 
In particular, a conjunction that contains both $A$ and $B$ in positive form can never be 
accepted under the maximum entropy model (the degree of belief is bounded from above by $0.25$).

Now consider a support relation $(A, B) \in \supports$.
Here, the most critical case is that $A$ is accepted, but $B$ is not.
Therefore, we consider a conjunction that contains $A \wedge \neg B$ as a subformula now.
In general, we have $P(A \wedge \neg B \wedge \chi) = P(A) \cdot P(\neg B \mid A) \cdot P(\chi \mid A \wedge \neg B)$. 
Now, we should expect $P(\neg B \mid A) \leq 0.5$ or equivalently
$P(B \mid A) > 0.5$ (if $A$ is accepted,
B should be accepted as well).
So the probability of the conjunction should again be bounded from above by $0.5$.
Of course, we have again $P^*(\neg B \mid A) = P^*(\neg B)$ under the maximum entropy model, so that
the probability of the conjunction may be too large when $P^*(\neg B) > 0.5$ and
we employ no constraints.
In this case, support-coherence gives us plausible guarantees.
\begin{proposition}
Let $A, B \in \args$ be arguments such that $(A, B) \in \supports$
and let $\chi$ be a conjunction of arguments from $\args$.
If $P^*$ satisfies support-coherence (S-COH), then
$P^*(A \wedge \neg B \wedge \chi) \leq \min \{0.25, 1 - P^*(A)\}$.
\end{proposition}
\begin{proof}
If $P^*$ satisfies support-coherence, we have $P^*(\neg B) = 1- P^*(B) \leq 1 - P^*(A)$.
Therefore,
$P^*(A \wedge \neg B \wedge \chi)  
= P^*(A) \cdot P^*(\neg B \mid A) \cdot P(\chi \mid A \wedge \neg B)
\leq P^*(A) \cdot (1 - P^*(A)) 
$. The claim follows now from the exact same analysis as in the previous proposition.
\end{proof}
The guarantees for support-coherence for supports are, of course, 
dual to those for coherence for attacks due to their dual nature 
and can be interpreted accordingly. 
Again, we can intuitively say that support relations have a classical meaning when the
source has probability $1$ and that their strength becomes gradually weaker as the
probability of the source moves towards $0$.

As a final simple example, let us consider the case that we have both an attack
$(A, C) \in \attacks$ and a support $(B, C) \in \supports$.
Then the chain rule implies that
$P(A^a \wedge B^b \wedge C^c) = P(A^a) \cdot P(B^b \mid A^a) \cdot P(C^c \mid A^a \wedge B^b)$,
where $a, b, c \in \{0,1\}$. Assuming that attack and
support can cancel their effects, we should expect something like 
\begin{enumerate}
	\item $P(C \mid A \wedge B) = P(C)$,
	\item $P(C \mid \neg A \wedge \neg B) = P(C)$, 
	\item $P(C \mid A \wedge \neg B) \leq 0.5$ and
	\item $P(\neg C \mid \neg A \wedge B) \leq 0.5$
\end{enumerate}
for the conditional probabilities.
The first two desiderata are always met by $P^*$ due to the independency assumptions.
Of course, this is just coincidential and the more interesting desiderata are the third and fourth 
one. 
When we apply both coherence and support-coherence, 
we maintain our previous guarantees for attacks and supports,
but also get a reasonable interaction between the belief in $A$ and $B$.
Namely, $A$ and $B$ cannot be accepted simultaneously. Arguably, this makes sense
because they give opposite evidence for $C$.
\begin{proposition}
Let $A, B, C \in \args$ be arguments such that $(A, C) \in \attacks$ and $(B, C) \in \supports$
and let $\chi$ be a conjunction of arguments from $\args$.
If $P^*$ satisfies coherence (COH) and support-coherence (S-COH), then
\begin{enumerate}
  \item $P^*(A) + P^*(B) \leq 1$,
	\item $P^*(A \wedge \neg B \wedge C \wedge \chi) \leq \min \{0.25, 1 - P^*(A)\}$.
	\item $P^*(\neg A \wedge B \wedge \neg C \wedge \chi) \leq \min \{0.25, 1 - P^*(B)\}$.
\end{enumerate}
\end{proposition}
\begin{proof}
1. We have $P^*(B) \leq P*(C) \leq 1 - P^*(A)$, where the first inequality follows from
support-coherence and the second from coherence. Adding $P^*(A)$ to both sides of the inequality yields the claim.

2. We have $P^*(A \wedge \neg B \wedge C \wedge \chi) 
= P^*(A) \cdot (1 - P^*(B)) \cdot P^*(C) \cdot P^*(\chi)
\leq P^*(A) \cdot P^*(C) \leq P^*(A) \cdot (1- P^*(A))$.
Now the two bounds can be derived exactly as before.

3. The proof is analogous to the proof of item 2.
\end{proof}
In summary, we demonstrated that applying constraints like coherence and support-coherence can give
us plausible guarantees for some probabilities under the maximum entropy model despite its 
independency assumptions. 
However, overall, the independency assumptions can cause difficulties. 
Most importantly, conditioning becomes meaningless under $P^*$.
However, sometimes we can work around this by simulating conditioning with constraints
as explained at the beginning of this section.
The independency assumptions also affect conjunctive queries. 
In this case, our semantical constraints may still give plausible guarantees
for the probabilities under $P^*$. We illustrated this for coherence and support-coherence.  
Still, in general, the probabilities may be larger (in case of attack) or lower (in case of support)
as desired.
A more detailed analysis of the possible cases would be interesting, but is out of scope here.
For this paper, our conclusion is that the principle of maximum entropy should only be applied with
care. It allows answering conjunctive queries efficiently, but it has to be checked carefully whether
the applied constraints can assure meaningful probabilities for the application at hand. 

\section{Related Work}

As mentioned in the introduction, there is a large variety of other probabilistic argumentation frameworks \cite{dung2010towards,li2011probabilistic,rienstra2012towards,hunter2014,doder2014probabilistic,polberg2014probabilistic,thimm2017probabilities,KidoO17,rienstra2018probabilistic,ThimmCR18,riveret2018labelling}. We sketch three early works here to give an impression of some ideas. \cite{dung2010towards} consider probability functions over possible worlds as well, but the mechanics
are very different from what we saw here. Instead of considering all possible probability functions that satisfy particular
constraints, a single probability function is derived from a set of probabilistic rules. Roughly speaking, these rules express the likelihood of
assumptions under given preconditions. Multiple rules for one assumption are only allowed if they can be ordered by specificity. \cite{li2011probabilistic} consider functions that assign probabilities to arguments (like probability labellings) and attack relations. The functions are supposed to be given
and allow assigning a probability to subgraphs of the given argumentation framework using common
independency assumptions. Then the probability of an argument is defined by taking the probability 
of every subgraph and adding those probabilities for which the argument is accepted in the subgraph
under a particular semantics. Since the number of subgraphs is exponential, the authors present a 
Monte-Carlo algorithm to approximate the probability of an argument.
In \cite{rienstra2012towards}, probabilities are again introduced over possible worlds. 
Again, a single probability distribution is derived from rules. However, in contrast to \cite{dung2010towards},
these rules are probabilistic extensions of a light form of ASPIC rules \cite{prakken2010abstract,caminada2007evaluation}. They are also more a flexible in that they do not need to be ordered according to specificity.

\cite{riveret2018labelling} recently introduced a very general probabilistic argumentation framework
that generalizes many ideas that have been considered before in the literature. 
The authors consider probability
functions over subsets of defeasible theories or over subgraphs. 
The latter approach 
can then be seen as a generalization of the former, which abstracts from the structure of arguments. 
The authors discuss probabilistic labellings that should not be confused with probability labellings that we considered here.
Roughly speaking, in \cite{riveret2018labelling}, a probabilistic labelling frame corresponds to
a probability function over subsets of possible classical labellings over an argumentation framework.
These probabilistic labelling frames can then be used to assign probabilities to arguments.
In this sense, a probabilistic labelling considered in \cite{riveret2018labelling} induces a
probability labelling as considered here. However, the focus in \cite{riveret2018labelling} is on 
conceptual questions and computational problems are not discussed. 

Our polynomial-time algorithms are based on a connection between probability functions
and probability labellings. The relationship is established by considering an equivalence relation
over probability functions. Conceptual similar ideas have been considered in probabilistic-logical 
reasoning. However, in this area, equivalence relations are introduced over possible worlds. 
Roughly speaking, the possible worlds are partitioned into equivalence classes that interpret the formulas that
appear in the knowledge base in the same manner. Reasoning algorithms can then
be modified to work on probability functions over equivalence classes
\cite{fischer1996tabl,kern2004combining,finthammer2012using,potyka2016solving}.
If the number of equivalence classes is small, a significant speedup can be obtained.
However, identifying compact representatives for these equivalence classes 
is intractable in general \cite{potyka2015concept}. In particular, in general,
the number of equivalence classes over possible worlds can still be exponential.
Indeed, many polynomial cases that we found here cannot be solved in polynomial time
with this approach. For example, if one atomic constraint $P(A)=p_A$ is given
for every argument $A$, every equivalence class of possible worlds will contain 
exactly one possible world, so that actually nothing is gained.
 
\cite{HunterT16} also considered an inconsistency-tolerant generalization of the entailment problem
that still works when there are conflicts between the partial probability assignment constraints and
the semantical constraints. We can probably derive similar polynomial runtime guarantees for this
problem. However, the approach in \cite{HunterT16} is based on the assumption that the semantical
constraints are consistent. This is no problem for the semantical constraints considered in 
\cite{HunterT16} because the probability of attacked arguments is only bounded from above
and the probability of non-attacked arguments is only bounded from below. However, in bipolar 
argumentation frameworks, we want to consider more complicated relationships and the constraints
can easily become inconsistent. Therefore, it is interesting to also analyze other variants
that use ideas for paraconsistent probabilistic reasoning \cite{Daniel:2009,potykaThimm2015} or reasoning with priorities \cite{potyka2015reasoning}.
It is also interesting to note that
our satisfiability test from Proposition \ref{prop_sat_polynomial_time} actually corresponds to an inconsistency
measure. If the knowledge base is inconsistent, the returned value will be $0$,
otherwise it measures by how much probability functions must violate the constraints numerically \cite{Potyka:2014}.

\section{Discussion and Future Work}

We showed that the satisfiability and entailment problem for the epistemic probabilistic argumentation
approach considered in \cite{HunterT16} can be solved in polynomial time. In fact, arbitrary 
linear atomic constraints can be considered.  
We found that the constraint language cannot be extended significantly without loosing polynomial runtime guarantees. 
However, it may still be possible to allow disjunctions of two probability statements, which would allow expressing conditional constraints like RAT.
For the query language, we found that conjunctive queries can still be answered in polynomial
time under the principle of maximum entropy when all constraints are atomic. 
However, atomic constraints cannot enforce stochastical dependence, so that the maximum entropy
model assumes stochastical independency between all arguments. While it is important to be aware
of this assumption, sometimes atomic constraints are sufficient to enforce meaningful
relationships between connected arguments as we demonstrated with coherence and support-coherence.
An interesting question for future work 
is whether we can compute conjunctive queries for the entailment problem in polynomial time
even without using the principle of maximum entropy. 

We focussed mainly on complexity results and did not speak much about the runtime guarantees
of our convex programming formulations. In general, interior-point methods can solve convex programs 
in cubic time in the number of optimization variables and optimization constraints \cite{Boyd2004}.
This means that all convex programs that we introduced here can be solved in cubic time in the size of the 
argumentation problem in the worst-case. 
Our linear programs for satisfiability and entailment can often be solved faster by using the Simplex algorithm.
Even though the Simplex algorithm has exponential worst-case runtime, in practice, the runtime usually depends
only linearly on the number of optimization variables and quadratically on the number of constraints \cite{matousek2007}.

Implementations for satisfiability and entailment can be found in the Java-library ProBabble\footnote{\url{https://sourceforge.net/projects/probabble/}}. You have to install IBM CPLEX in
order to use ProBabble, but IBM offers free licenses for academic purposes. Problems with thousands of arguments can usually be solved 
within a few hundred milliseconds. Without the labelling approach, the same amount of time would be needed for 10-15 arguments
already because the number of possible worlds grows exponentially.

\subsubsection*{Acknowledgements:} I am very grateful to Anthony Hunter and Sylwia Polberg for helpful discussions. I am also indebted to some anonymous reviewers for their critical comments and for 
pointing out that the independency assumptions of the maximum entropy model should be analyzed
carefully.

\bibliographystyle{ACM-Reference-Format}
\bibliography{references}

\end{document}